\documentclass[10pt,twocolumn,journal]{IEEEtran}

\usepackage{xcolor}
\usepackage{amsmath,amsthm,amssymb,mathtools,bm,amsfonts}
\mathtoolsset{mathic=true}
\usepackage{bbm,dsfont}
\usepackage{scalerel}
\usepackage{graphicx}
\usepackage{pgfplots}
\pgfplotsset{compat=1.17}
\usepackage{tikz}
\usetikzlibrary{svg.path,shapes.geometric,arrows.meta,positioning}
\usepackage{psfrag}
\usepackage{subcaption}
\usepackage{booktabs,adjustbox,tabularx,multirow,array}
\newcolumntype{Y}{>{\centering\arraybackslash}X}
\usepackage{epstopdf}
\usepackage[inline,shortlabels]{enumitem}
\usepackage{cite}
\usepackage[colorlinks=true,allcolors=blue]{hyperref}
\usepackage{microtype}
\usepackage{stfloats}
\usepackage{cuted}
\usepackage{balance}
\usepackage{pifont}

\setlist{noitemsep,leftmargin=*,topsep=2pt,partopsep=1pt}
\setlist[itemize]{itemsep=0pt,topsep=0pt,parsep=0pt,partopsep=0pt}
\setlist[enumerate]{itemsep=0pt,topsep=0pt,parsep=0pt,partopsep=0pt}
\renewcommand{\arraystretch}{0.95}
\newcommand{\EX}{\mathbb{E}}

\newcommand{\SUBJTO}{\mathrm{s.t.}}
\newcommand{\epsstale}{\epsilon_{\mathrm{stale}}}
\newcommand{\epsdir}{\epsilon_{\mathrm{dir}}}
\newcommand{\best}[1]{\textbf{#1}}
\newcommand{\cmark}{\ding{51}}
\newcommand{\xmark}{\ding{55}}
\newcommand{\pmark}{\ensuremath{\triangle}}
\newcommand{\blue}[1]{{\color{blue}#1}}
\colorlet{linkblue}{blue}\hypersetup{allcolors=linkblue}\colorlet{blue}{black}\renewcommand{\blue}[1]{#1}
\newtheorem{theo}{Theorem}
\newtheorem{lem}{Lemma}
\newtheorem{propo}{Proposition}
\theoremstyle{definition}

\definecolor{orcidlogocol}{HTML}{A6CE39}
\tikzset{
  orcidlogo/.pic={
    \fill[orcidlogocol] svg{M256,128c0,70.7-57.3,128-128,128C57.3,256,0,198.7,0,128C0,57.3,57.3,0,128C0,57.3,57.3,0,128,0C198.7,0,256,57.3,256,128z};
    \fill[white] svg{M86.3,186.2H70.9V79.1h15.4v48.4V186.2z}
                 svg{M108.9,79.1h41.6c39.6,0,57,28.3,57,53.6c0,27.5-21.5,53.6-56.8,53.6h-41.8V79.1z M124.3,172.4h24.5c34.9,0,42.9-26.5,42.9-39.7c0-21.5-13.7-39.7-43.7-39.7h-23.7V172.4z}
                 svg{M88.7,56.8c0,5.5-4.5,10.1-10.1,10.1c-5.6,0-10.1-4.6-10.1-10.1c0-5.6,4.5-10.1,10.1-10.1C84.2,46.7,88.7,51.3,88.7,56.8z};
  }
}
\newcommand\orcidicon[1]{\href{https://orcid.org/#1}{\mbox{\scalerel*{
\begin{tikzpicture}[yscale=-1,transform shape]
\pic{orcidlogo};
\end{tikzpicture}
}{|}}}}

\title{Secure Decentralized Federated Learning\\ via Gossip and Virtual Voting}
\author{
Amirhossein Taherpour \orcidicon{0000-0003-4647-102X}, 
and Xiaodong Wang \orcidicon{0000-0002-2945-9240},~\IEEEmembership{Fellow,~IEEE}
\thanks{Amirhossein Taherpour and Xiaodong Wang are with the Department of Electrical Engineering, Columbia University, New York, NY, USA (e-mails: at3532@columbia.edu, xw2008@columbia.edu).}
}
\date{}

\begin{document}
\maketitle
\thispagestyle{empty}

\begin{abstract}
\blue{
Decentralized federated learning (DFL) removes the central server by letting nodes exchange model updates through peer-to-peer gossip, but existing gossip-based methods often lack provenance finality and resilience to Byzantine or lazy participants. Ledger-assisted federated learning (FL) improves auditability, yet blockchains, shards, or settlement committees can reintroduce global coordination costs that conflict with DFL locality. This paper proposes \emph{gspDAG-FL}, a secure DFL framework that derives consensus from the same gossip history used to disseminate models. Nodes exchange model payloads only with neighbors, while full nodes collect event certificates and receiver-endorsed accepted gossip proofs, reconstruct a compact Topology directed acyclic graph (DAG), and run Hashgraph-style virtual voting followed by compact full-node certificates. Finality is over unique model-origin tuples, not identical local parameter states. To improve resilience, gspDAG-FL combines payload validation, accepted-proof validation, and private semantic audit before aggregation. We formalize the adversarial setting, prove safety and conditional liveness of the control plane, and give a convergence guarantee for certified perturbed gossip under time-varying effective mixing. Experiments on MNIST classification and Penn Treebank language modeling, using fair held-out validation/audit data and networks up to \(N=100\), show that gspDAG-FL achieves learning quality close to validation-based ledger FL while reducing coordination bottlenecks, improving throughput, and maintaining high invalid-origin detection under mixed Byzantine and lazy participation.
}
\end{abstract}

\begin{IEEEkeywords}
Decentralized federated learning (DFL), gossip, Hashgraph, directed acyclic graph (DAG), Byzantine fault tolerance (BFT), virtual voting, federated learning security, auditability.
\end{IEEEkeywords}

\section{Introduction}
\begingroup
\color{blue}

Federated learning (FL) trains a model across distributed devices or institutions while keeping raw data local and exchanging only model-side information~\cite{McMahan2017,Kairouz2021,Taherpour2026ZKHybridFL}. In the standard architecture, a central server coordinates local training, aggregation, and model redistribution, often through federated averaging and secure aggregation~\cite{Bonawitz2017}. This architecture is effective, but it concentrates coordination, trust, and failure modes at the server. It may also amplify privacy risk, since gradients and model updates can leak sensitive information even when raw data are never shared~\cite{Zhu2019}.\par

A natural alternative is decentralized federated learning (DFL), where peers exchange updates directly through a communication graph. Randomized gossip and decentralized stochastic optimization show that repeated neighbor averaging can converge at a rate controlled by topology and mixing quality~\cite{Boyd2006,Lian2017}. Asynchronous push--pull, push-sum, and compressed gossip variants improve wall-clock efficiency or communication cost under heterogeneous links~\cite{Lian2018,Assran2019,Koloskova2019}. Recent systems such as GossipFL, FedDual, and semi-decentralized optimization further show that neighbor-to-neighbor communication can reduce central bottlenecks~\cite{Tang2023,Chen2023,Wang2025,Taherpour2025SPIDChain}. However, most serverless gossip-based methods assume benign participation. They usually do not provide a system-wide record of update provenance, nor do they define which model origins are final and eligible for aggregation when some nodes are Byzantine, lazy, or behaviorally malicious.\par

Robust aggregation addresses part of this problem by limiting the effect of abnormal updates during aggregation. Byzantine-resilient rules and coordinate-wise robust methods improve resilience under bounded adversarial fractions~\cite{Blanchard2017,Yin2018,Pillutla2022}. Sybil-aware defenses reduce the impact of multiple colluding identities~\cite{Fung2018}. Backdoor attacks show, however, that malicious clients may preserve clean accuracy while inducing attacker-chosen behavior on triggered inputs~\cite{Bagdasaryan2020,Wang2020}. Post-hoc and input-level backdoor detectors, including Neural Cleanse, STRIP, and spectral signatures, study complementary forms of behavioral inspection~\cite{Wang2019a,STRIP2019,SpectralSignatures2018}. These defenses are important, but they are not a complete substitute for finality in DFL: a node also needs to know whether an update origin has sufficient provenance and whether the network agrees that it should be eligible for aggregation.\par

A second line of work introduces distributed ledgers~\cite{Taherpour2025CodedBlockchainIoT,Taherpour2024HybridChain} into FL. Blockchain-assisted FL systems use tamper-evident records, smart contracts, or consensus mechanisms to reduce reliance on a trusted coordinator~\cite{Kim2019,Warnat2021,Qu2020,Li2022,jkmn1,bambool1,WWWZ1}. DAG and sharded ledgers reduce some serialization costs by allowing concurrent approvals or shard-level processing, as in ChainFL, DAG-FL, DAG-EnseFL, and DAG-BFL~\cite{Yuan2024,Cao2023,Chen2025,Wang2025DAG}. IronForge further studies open and fair decentralized FL through committee-based validation and settlement~\cite{BB6}. These systems improve auditability, but in most cases the ledger remains a separate coordination layer: model updates are submitted to a block, shard, DAG ledger, or committee, and that external mechanism decides admission.\par

The gap is therefore not simply the absence of a DAG ledger. The gap is the lack of a gossip-native consensus layer for DFL. In serverless FL, scalability comes from locality: each node exchanges information with a small neighborhood. If finality requires broad model dissemination, complete ledger visibility at all nodes, or a separate global committee decision for every update, then the protocol partially reintroduces the coordination bottleneck that gossip was meant to avoid. Hashgraph-style consensus suggests a different direction: the communication history itself can form a DAG, and virtual voting can infer local confirmation from gossip-about-gossip metadata~\cite{Baird2016}. Classical Byzantine fault-tolerant consensus relies on quorum intersection for safety~\cite{Castro1999}; Hashgraph adapts this principle to a gossip DAG. The question is how to adapt this principle to FL so that the control plane certifies provenance-admissible model origins while the data plane remains local and peer-to-peer.\par

\begin{table*}[!t]
\caption{Feature-level positioning of gspDAG-FL relative to representative decentralized and ledger-assisted FL systems. Here \(\pmark\) denotes partial support.}
\label{tab:related_positioning}
\centering
\footnotesize
\setlength{\tabcolsep}{3.2pt}
\renewcommand{\arraystretch}{1.13}
\begin{adjustbox}{max width=\textwidth}
\begin{tabular}{p{0.22\textwidth}|c|c|p{0.12\textwidth}|p{0.09\textwidth}|p{0.11\textwidth}|p{0.12\textwidth}}
\toprule
\textbf{System family} &
\textbf{Local payload} &
\textbf{Gossip-derived finality} &
\textbf{Finality object} &
\textbf{Vote type} &
\textbf{Validation} &
\textbf{Bottleneck}
\\
\midrule
Serverless gossip:
D-PSGD, AD-PSGD, GossipFL, FedDual~\cite{Lian2017,Lian2018,Tang2023,Chen2023}
& \cmark & \xmark & None & None & Local / none & No finality\\
\midrule
Blockchain FL:
BLADE-FL~\cite{Li2022}
& \xmark & \xmark & Block update & PoW & Public & Blocks\\
\midrule
DAG / sharded ledger FL:
ChainFL, DAG-FL, DAG-EnseFL, DAG-BFL~\cite{Yuan2024,Cao2023,Chen2025,Wang2025DAG}
& \pmark & \xmark & Ledger update & Ledger / shard & Public & Shard sync\\
\midrule
Committee-based open FL:
IronForge~\cite{BB6}
& \pmark & \xmark & Committee model & Committee & Public / PoL & Settlement\\
\midrule
\textbf{gspDAG-FL}
& \cmark & \cmark & \textbf{Origin tuple} & \textbf{Virtual + cert.} & \textbf{Local/private} & \textbf{Proof metadata}\\
\bottomrule
\end{tabular}
\end{adjustbox}
\end{table*}

This paper proposes \emph{gspDAG-FL}, a secure DFL framework that derives finality from the same gossip history used to disseminate models. Nodes exchange model payloads only through one-hop neighbor gossip. In parallel, compact signed event certificates and receiver-endorsed accepted gossip proofs are forwarded to full nodes, which reconstruct a Topology DAG and run virtual voting over communication history. The resulting finality is over unique model-origin tuples, not over identical local parameter states. Each node aggregates only the certified models that it has locally observed and stored, so locality is preserved while invalid origins are filtered through a global provenance-admissibility decision.\par

The key distinction is the source of finality. Existing ledger-assisted FL systems finalize submitted updates through block, shard, DAG-ledger, or committee state. In contrast, gspDAG-FL finalizes model-origin tuples from signed gossip-history proofs. Thus, model tensors remain on local peer-to-peer data paths, while full nodes run virtual voting over compact topology metadata and exchange only compact confirmation certificates.\par

The contributions are as follows.
\begin{enumerate}
\item We introduce a gossip-native DFL architecture in which the data plane and control plane are separated. Model tensors remain on local neighbor-to-neighbor paths, while full nodes process event certificates and accepted-proof metadata. Consensus is therefore derived from observed gossip history rather than imposed by a separate global block, shard, or committee protocol.

\item We formalize the adversarial setting for decentralized FL with learning-clean, lazy, Byzantine, and control-correct nodes. We derive the local update rule from a consensus-constrained optimization view, state the required mixing and smoothness assumptions, and clarify that global finality concerns unique origin tuples rather than identical per-node parameters.

\item We design a multi-stage admission pipeline. Payload validation rejects stale, abnormally large, or directionally inconsistent updates before forwarding; accepted-proof validation prevents forged topology edges and equivocation; and post-consensus private semantic audit removes behaviorally anomalous confirmed models before aggregation.

\item We establish theoretical properties of the control and learning planes. Under stated quorum and delivery assumptions, virtual voting is well-defined, full-node certificates satisfy safety, and termination follows once enough valid origins are disseminated. We also give a convergence statement for certified perturbed gossip under time-varying effective mixing.

\item We evaluate gspDAG-FL on image classification and language modeling under Byzantine and lazy participation. The experiments use the same effective training budget for all methods, include network sizes up to \(N=100\), and report learning quality, detection rates, ripple dynamics, latency, throughput, and convergence rounds.
\end{enumerate}

The remainder of the paper is organized as follows. Sec.~II reviews centralized, decentralized gossip-based, and ledger-assisted FL protocols, and introduces the adversarial and optimization setting. Sec.~III presents the gspDAG-FL architecture, including light-node and full-node roles, the Topology DAG, and the epoch workflow. Sec.~IV specifies validation, virtual voting, finality, robustness, complexity, and the learning guarantee. Sec.~V reports the simulation setup and results. Sec.~VI concludes the paper. The appendices contain the technical proofs.

\endgroup

\section{Federated Learning Protocols}

\subsection{Conventional FL}
Suppose that \(N\) nodes participate in FL. Node \(i\) holds local data \(\mathcal{D}_i\) and defines
\begin{equation}
F_i(\theta)\triangleq
\mathbb{E}_{x\sim\mathcal{D}_i}
\left[\mathcal{L}_i(x;\theta)\right],
\end{equation}
where \(\theta\) is the model parameter. In the benign setting, the goal is
\begin{equation}\label{main}
    \min_{\theta}F(\theta)
    \triangleq
    \frac{1}{N}\sum_{i=1}^{N}F_i(\theta).
\end{equation}

\blue{
We distinguish learning behavior from control-plane behavior. Let \(\mathcal{H}_{\mathrm{L}}\), \(\mathcal{Z}\), and \(\mathcal{B}\) denote the learning-clean, lazy, and Byzantine node sets, respectively. Learning-clean nodes perform fresh local training and follow the protocol. Lazy nodes are control-correct but may replay stale updates. Byzantine nodes may send arbitrary payloads or equivocate, subject only to cryptographic authentication. The control-correct set is
\[
    \mathcal{C}_{\mathrm{ctrl}}
    =
    \{1,\ldots,N\}\setminus\mathcal{B}.
\]
The learning target is the learning-clean objective
\begin{equation}\label{eq:honest_objective}
    \min_{\theta}F_{\mathcal{H}_{\mathrm{L}}}(\theta)
    \triangleq
    \frac{1}{|\mathcal{H}_{\mathrm{L}}|}
    \sum_{i\in\mathcal{H}_{\mathrm{L}}}F_i(\theta).
\end{equation}
The protocol does not know these sets. Its validation and consensus layers aim to certify enough fresh, non-Byzantine origins while excluding stale or inconsistent origins.
}

In conventional FL~\cite{McMahan2017,Bonawitz2017}, as shown in Fig.~\ref{fornow}(a), a central server coordinates all nodes. At epoch \(t\), node \(i\) starts from the current global model \(\theta_i^{t,0}=\theta^{t-1}\) and performs \(K\) local stochastic-gradient steps:
\begin{equation}\label{eq:mini-batch}
 \theta_i^{t,k+1}
 =
 \theta_i^{t,k}
 -
 \eta
 \nabla\mathcal{L}_i
 \left(x_i^{t,k};\theta_i^{t,k}\right),
 \quad k=0,\ldots,K-1.
\end{equation}
The server aggregates
\begin{equation}\label{eq:central_agg}
  \theta^t
  =
  \sum_{i=1}^{N}w_i\theta_i^{t,K},
  \quad
  w_i>0,
  \quad
  \sum_{i=1}^{N}w_i=1,
\end{equation}
and broadcasts \(\theta^t\) to all nodes.

\subsection{Gossip-based Decentralized FL}
\label{gossip-fl}
Gossip-based DFL removes the central server by allowing each node to exchange models only with a neighbor set. Let \(\mathcal{G}=(\mathcal{V},\mathcal{E})\) be the communication graph, \(\mathcal{N}_i\) the neighbors of node \(i\), and \(\mathcal{N}_i^+=\mathcal{N}_i\cup\{i\}\).

\blue{
The consensus-constrained optimization form underlying decentralized training is
\begin{equation}\label{eq:consensus_problem}
    \min_{\{\theta_i\}_{i=1}^N}
    \sum_{i=1}^{N}F_i(\theta_i)
    \quad
    \SUBJTO
    \quad
    \theta_i=\theta_j,\ \forall (i,j)\in\mathcal{E}.
\end{equation}
The update used below can be interpreted as a linearized proximal primal--dual step for \eqref{eq:consensus_problem}. At epoch \(t\), node \(i\) approximately minimizes
\begin{equation}\label{eq:local_surrogate}
    F_i(\theta)
    -
    \langle \hat g_i^{t-1},\theta\rangle
    +
    \frac{1}{2\lambda}
    \|\theta-\theta_i^{t-1}\|_2^2,
\end{equation}
where \(\hat g_i^{t-1}\) is a local dual/tracking variable and \(\lambda>0\) controls the proximal consensus pull.
}

Node \(i\) initializes \(\theta_i^{t,0}=\theta_i^{t-1}\) and performs
\begin{equation}
\label{SL_main}
\begin{aligned}
\theta_i^{t,k+1}
={}&
\theta_i^{t,k}
-
\eta
\Bigg(
\nabla \mathcal{L}_i
\big(x_i^{t,k};\theta_i^{t,k}\big)
-
\hat{g}_i^{t-1}
\\
&\qquad
+
\frac{1}{\lambda}
\bigl(\theta_i^{t,k}-\theta_i^{t-1}\bigr)
\Bigg),
\\
&\hspace{35mm}
k=0,\ldots,K-1 .
\end{aligned}
\end{equation}
After \(K\) local steps,
\begin{equation}\label{eq:dual_update}
   \hat{g}_i^t
   =
   \hat{g}_i^{t-1}
   -
   \frac{1}{\lambda}
   \bigl(\theta_i^{t,K}-\theta_i^{t-1}\bigr),
\end{equation}
and
\begin{equation}\label{SL_intermediate}
  z_i^t
  =
  \theta_i^{t,K}
  -
  \lambda\hat{g}_i^{t-1}.
\end{equation}
A one-hop gossip aggregation would then take the form
\begin{equation}\label{eq:gossip_agg}
  \theta_i^t
  =
  \sum_{j\in \mathcal{N}_i^+}w_{ij}z_j^t,
  \quad
  w_{ij}\ge0,
  \quad
  \sum_{j\in \mathcal{N}_i^+}w_{ij}=1.
\end{equation}

\blue{
For nominal decentralized convergence, the graph is connected, \(w_{ij}=0\) if \(j\notin\mathcal{N}_i^+\), and \(W=[w_{ij}]\) is doubly stochastic:
\begin{equation}\label{eq:doubly_stochastic}
    W\bm{1}=\bm{1},
    \qquad
    \bm{1}^{\top}W=\bm{1}^{\top}.
\end{equation}
Its disagreement factor satisfies
\begin{equation}\label{eq:spectral_gap}
    \rho
    \triangleq
    \left\|
    W-\frac{1}{N}\bm{1}\bm{1}^{\top}
    \right\|_2
    <1.
\end{equation}
We also assume that learning-clean losses are \(L\)-smooth, stochastic gradients have bounded variance, and data heterogeneity is bounded:
\begin{equation}\label{eq:data_heterogeneity}
    \frac{1}{|\mathcal{H}_{\mathrm{L}}|}
    \sum_{i\in\mathcal{H}_{\mathrm{L}}}
    \left\|
    \nabla F_i(\theta)
    -
    \nabla F_{\mathcal{H}_{\mathrm{L}}}(\theta)
    \right\|_2^2
    \le \zeta^2 .
\end{equation}
In gspDAG-FL, the final aggregation matrix is time-varying because certified and post-audit origin sets differ across nodes; this is handled explicitly in Sec.~\ref{subsec:learning_effect}.
}

\subsection{FL over Blockchain}
\label{blockchain_based}
Blockchain integration provides a tamper-evident ledger for recording updates and a consensus mechanism for selecting accepted updates~\cite{Qu2020,Li2022}. In each epoch, after local training using \eqref{eq:mini-batch}, node \(i\) signs and broadcasts \(\theta_i^{t,K}\) as a transaction. Other nodes or miners validate received updates, often using a public validation set. Accepted updates are included in a block, and the block records both accepted transactions and an aggregate. This improves auditability relative to a trusted server, but global propagation, block construction, and confirmation delay introduce latency and throughput bottlenecks that are restrictive for gossip-native DFL.

\begin{figure}[htbp]
  \centering
  \begin{subfigure}[b]{0.32\columnwidth}
    \centering
    \includegraphics[width=\linewidth]{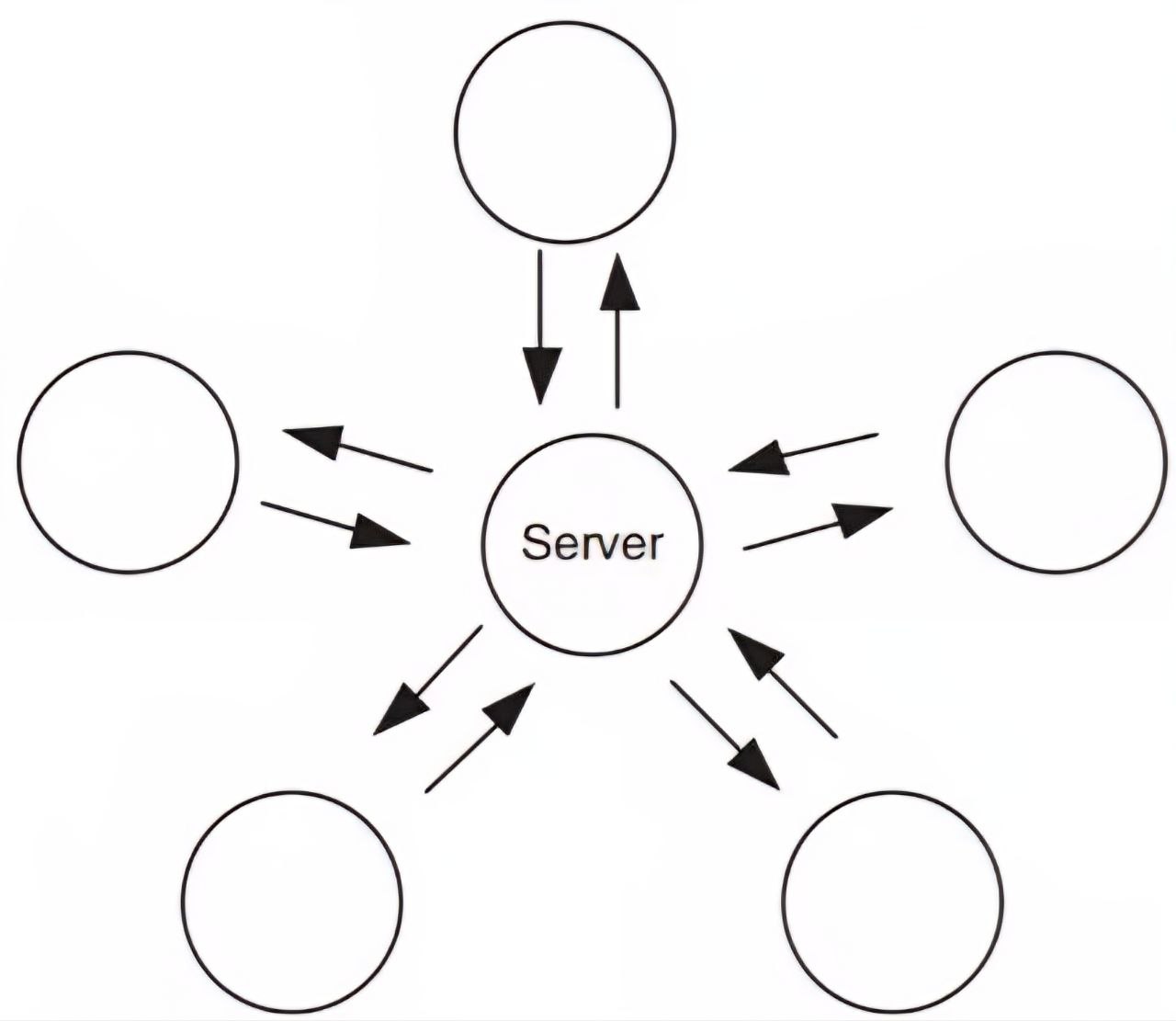}
    \caption{Traditional centralized FL}
    \label{fig:sub1}
  \end{subfigure}%
  \hfill
  \begin{subfigure}[b]{0.32\columnwidth}
    \centering
    \includegraphics[width=\linewidth]{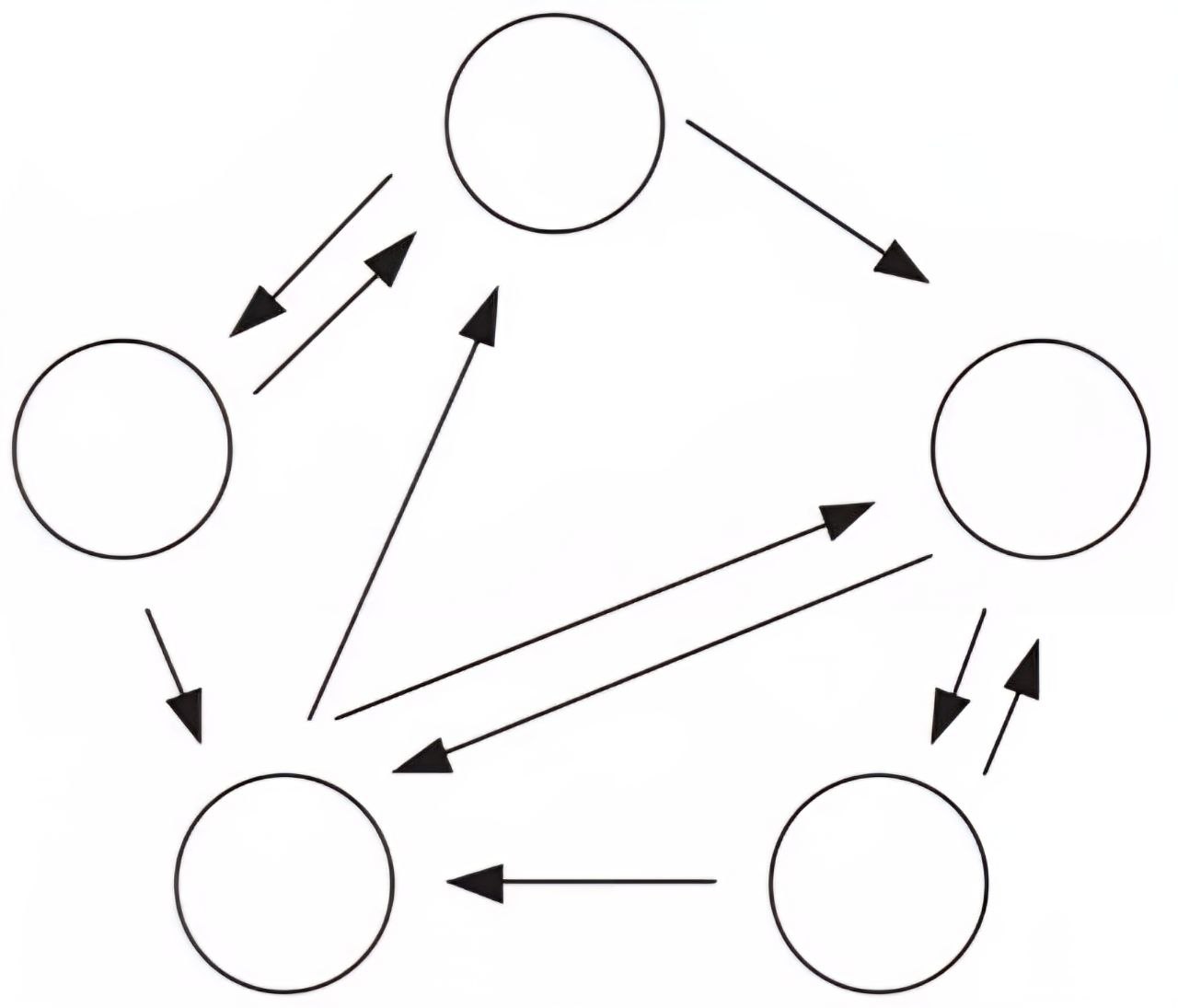}
    \caption{Gossip-based DFL}
    \label{fig:sub2}
  \end{subfigure}%
  \hfill
  \begin{subfigure}[b]{0.32\columnwidth}
    \centering
    \includegraphics[width=\linewidth]{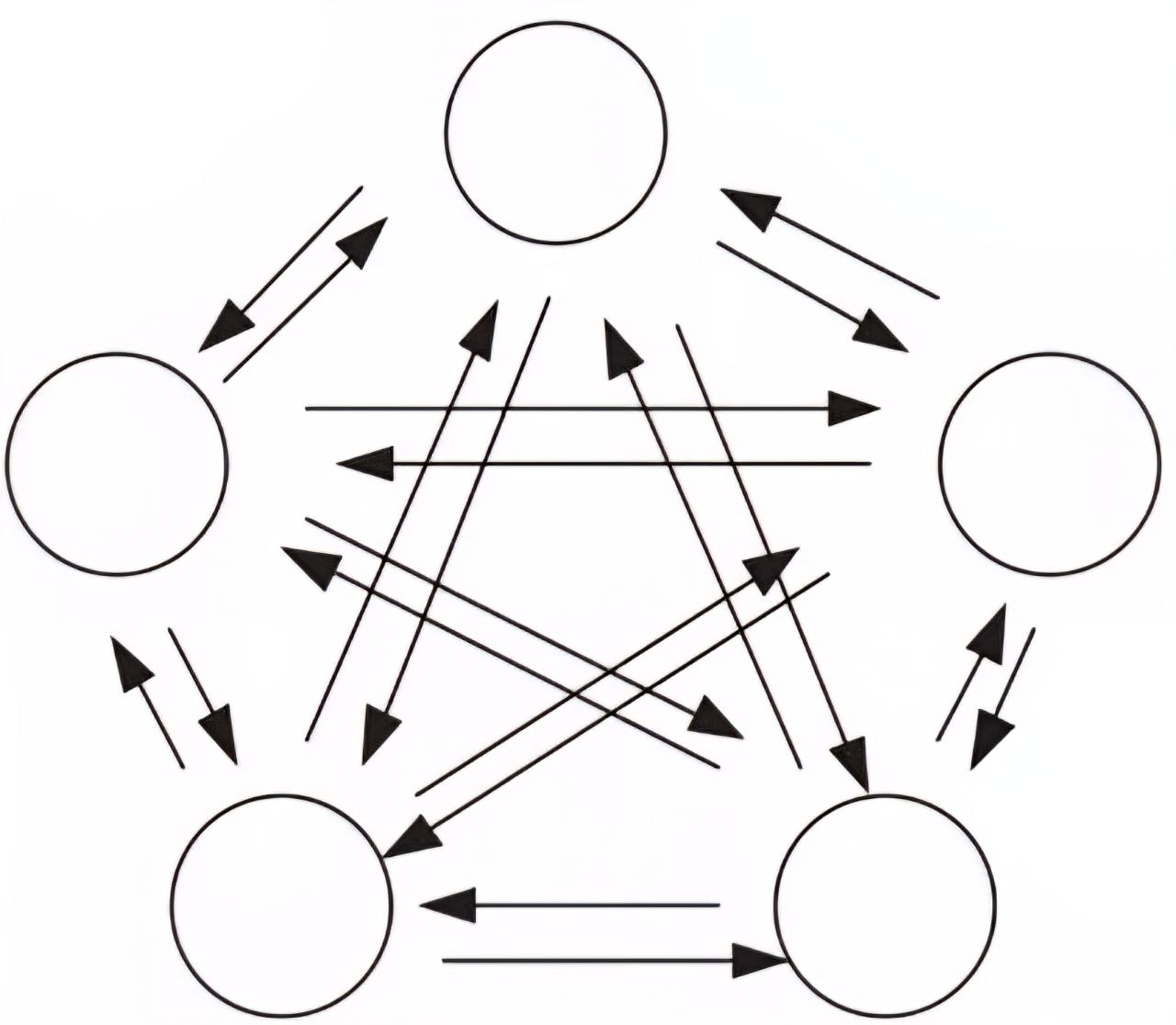}
    \caption{Blockchain-based FL}
    \label{fig:sub3}
  \end{subfigure}
  \caption{Communication protocols for different FL approaches.}
  \label{fornow}
\end{figure}

\section{Gossip-based FL Over Hashgraph DAG (gspDAG-FL)}
\subsection{System Architecture}
\label{sec:architecture}

In gspDAG-FL, all nodes train and gossip as light nodes, while a subset \(\mathcal{F}\subseteq\{1,\ldots,N\}\) also act as full nodes. Full nodes reconstruct a Topology DAG from compact certificates and run virtual voting over that DAG. The data channel carries model payloads between neighbors. The control channel carries event certificates, accepted gossip proofs, equivocation reports, confirmation vectors, and termination certificates.\par

\blue{
The control plane certifies provenance-admissible origin tuples. It does not certify that a model is semantically safe, and it does not force all local parameters to be identical. Semantic safety is checked locally after finality; aggregation eligibility is therefore the result of both control-plane certification and private post-consensus audit.
}

\subsubsection{Events, origin tuples, and edge convention}
Each epoch \(t\) is divided into ripples \(r=0,1,\ldots,R_t\). Node \(i\) creates one signed event certificate per ripple:
\[
    \chi_i^t(r)
    =
    \mathrm{Sig}_i
    \bigl(
    i,t,r,h_i^t(r-1),\mathrm{type},\mathrm{meta}
    \bigr),
\]
where \(h_i^t(r-1)=h(e_i^t(r-1))\) is the receiver self-parent hash. This certificate is sent to full nodes every ripple, including empty ripples. Thus, full nodes can reconstruct both non-empty events and empty heartbeat events. Empty events have a self-parent but no gossip parent.\par

The genesis event \(e_i^t(0)\) contains the corrected model \(z_i^t\). Its unique model-origin tuple is
\begin{equation}\label{eq:origin_tuple}
    \omega_i^t
    \triangleq
    \bigl(i,t,h(e_i^t(0)),h(z_i^t)\bigr).
\end{equation}
All finality decisions are over tuples \(\omega\), not over bare node IDs. If full nodes observe two conflicting genesis tuples with the same origin node and epoch, that origin is marked equivocated for epoch \(t\), and all conflicting tuples from that origin are excluded from certification.\par

For an accepted transmission from sender \(a\) to receiver \(i\) in ripple \(r\), the sender first signs a send proof
\[
    \pi_{a\to i}^{t,\mathrm{send}}(r)
\]
binding sender, receiver, epoch, ripple, sender-parent hash \(h(e_a^t(r-1))\), receiver self-parent hash \(h(e_i^t(r-1))\), origin tuple \(\omega\), and model hash. If node \(i\) verifies the proof, validates the payload, and has not already observed \(\omega\), it signs a receiver endorsement
\[
    \sigma_i^{t,\mathrm{recv}}(r)
\]
over the same metadata and an acceptance flag. The accepted proof is
\[
    \Pi_{a\to i}^t(r)
    =
    \bigl(
    \pi_{a\to i}^{t,\mathrm{send}}(r),
    \sigma_i^{t,\mathrm{recv}}(r),
    \chi_i^t(r)
    \bigr).
\]
A duplicate origin tuple is not counted again for readiness; the receiver creates a heartbeat event instead of a new origin-observation event. If the same receiver signs multiple distinct event certificates for the same epoch and ripple, full nodes mark the receiver as equivocating and discard the conflicting event certificates for that decision instance.\par

We use the parent-to-child edge convention. A self edge is
\[
    e_i^t(r-1)\longrightarrow e_i^t(r),
\]
and an accepted gossip-parent edge is
\[
    e_a^t(r-1)\longrightarrow e_i^t(r).
\]
Hence, all DAG edges increase the ripple index, so the graph is acyclic. Virtual voting uses ancestor reachability: event \(e\) sees event \(u\) if there is a directed path \(u\to^\star e\), equivalently if one reaches \(u\) by following parent links backward from \(e\).

\begin{figure}[!t]
    \centering
    \includegraphics[width=\columnwidth]{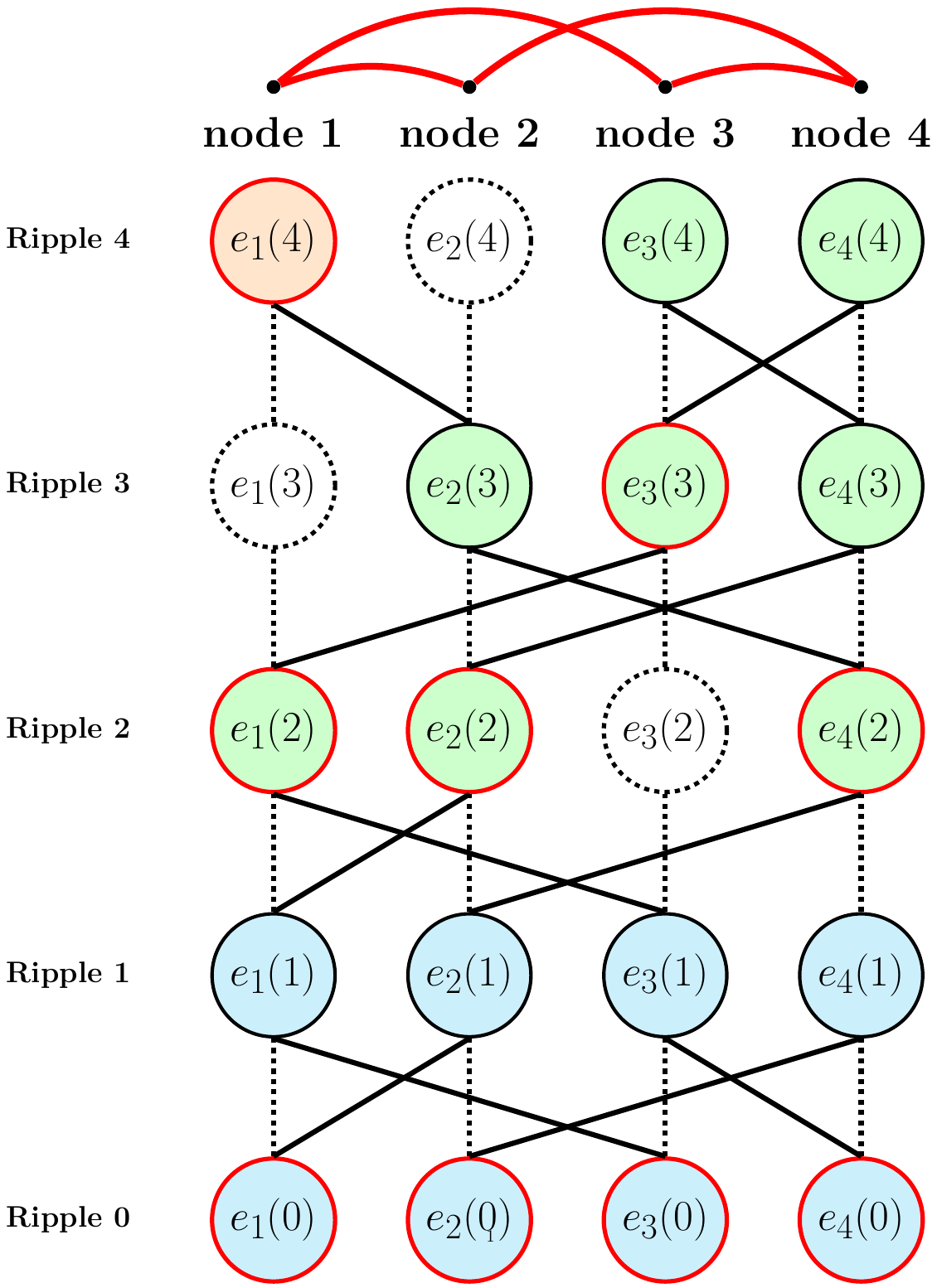}
    \caption{Illustration of the proposed Topology DAG with \(N=4\) nodes over five ripples. Cross-column solid edges denote accepted gossip-parent links; same-column dashed edges denote self-predecessor links. Colors indicate moment values. Red-outlined events are voting events, and dashed circles denote empty heartbeat events.}
    \label{dag}
\end{figure}

\subsection{Light-node Duties}

\subsubsection{Storage}
Each light node stores its own event chain, accepted payloads, observed origin tuples, and received parent payloads needed for validation, forwarding, audit, and aggregation. It does not store the full Topology DAG. Model tensors therefore remain distributed along data-plane paths rather than being replicated at every node.

\subsubsection{Communication}
At ripple \(r\ge1\), if node \(i\) has a non-empty event \(e_i^t(r-1)\), it selects a neighbor \(j\in\mathcal{N}_i\), sends the stored model payload and a fresh send proof, and then creates its own event certificate. If it has no payload to forward, it still creates and sends a heartbeat event certificate to full nodes. A receiver accepts at most one incoming payload per ripple under a deterministic tie rule. It rejects invalid, stale, duplicate, or conflicting payloads and creates a heartbeat event in that ripple.

\subsection{Full-node Duties}

Full nodes verify event certificates, send proofs, receiver endorsements, self-parent hashes, gossip-parent hashes, origin tuples, and duplicate/equivocation rules. Valid event certificates add vertices and self edges to the Topology DAG. Valid accepted proofs add gossip-parent edges. Invalid messages are ignored and, if equivocation is detected, the conflicting creator or origin is marked invalid for the corresponding epoch.\par

Full nodes run virtual voting locally over their reconstructed DAG prefixes. They then exchange compact signed confirmation vectors and termination certificates. Thus, virtual voting is used to infer local confirmation from DAG structure, while the full-node certificate exchange provides quorum finality.

\subsection{Workflow of gspDAG-FL}
\label{subsec:ripples}

Each epoch \(t\) has three stages.

\subsection*{\textbf{Stage 1: Local model update}}
Node \(i\) initializes \(\theta_i^{t,0}=\theta_i^{t-1}\), performs \eqref{SL_main}--\eqref{SL_intermediate}, obtains \(z_i^t\), creates genesis event \(e_i^t(0)\), and defines \(\omega_i^t\) from \eqref{eq:origin_tuple}.

\subsection*{\textbf{Stage 2: Gossip over ripples}}

\subsubsection*{\textsf{Step 1: Data-channel gossip}}
Each node with a non-empty previous event sends one payload and proof to one selected neighbor. The receiver checks the proof, payload hash, duplicate rule, equivocation rule, and payload-validation tests. If accepted, it signs a receiver endorsement and records the origin tuple; otherwise it records only a heartbeat event.

\subsubsection*{\textsf{Step 2: Control-channel topology sync}}
Every node sends its event certificate to full nodes. If it accepted a payload, it also sends the accepted proof. Full nodes verify the messages and update the Topology DAG.

\subsubsection*{\textsf{Step 3: Termination check}}
Full nodes run virtual voting over provenance-admissible origin tuples. Let \(C^t(r)\) be the set of control-plane certified origin tuples by the end of ripple \(r\). If enough nodes are ready, full nodes broadcast termination certificates. A termination certificate matches on the tuple \((t,r,C^t(r))\). DAG-prefix digests may be included for audit but are not required to match for termination, since honest full nodes can have different extra non-decisive proofs.

\subsection*{\textbf{Stage 3: Semantic audit and model aggregation}}
Let \(\mathcal{O}_i^t(r)\) be the set of origin tuples observed and stored by node \(i\). After termination, node \(i\) forms the certified-and-observed set
\[
    \mathcal{A}_i^t
    =
    C^t(r)\cap \mathcal{O}_i^t(r).
\]
It then applies the semantic audit in Sec.~\ref{sssec:semantic_audit}, producing
\[
    \widetilde{\mathcal{A}}_i^t\subseteq \mathcal{A}_i^t.
\]
For each \(\omega\in\widetilde{\mathcal{A}}_i^t\), let \(z_\omega\) denote the corresponding stored payload. Aggregation uses weights over certified observed origin tuples, not one-hop neighbor weights:
\begin{equation}\label{ours_agg}
    \theta_i^t
    =
    \sum_{\omega\in\widetilde{\mathcal{A}}_i^t}
    \tilde w_{i\omega}^t z_\omega,
    \qquad
    \tilde w_{i\omega}^t
    =
    \frac{\alpha_{i\omega}^t}
    {\sum_{\omega'\in\widetilde{\mathcal{A}}_i^t}\alpha_{i\omega'}^t}.
\end{equation}
Here \(\alpha_{i\omega}^t>0\) may encode trust, freshness, validation score, or uniform weighting; the simulations use uniform \(\alpha_{i\omega}^t=1\). If \(\widetilde{\mathcal{A}}_i^t=\varnothing\), node \(i\) keeps \(\theta_i^t=\theta_i^{t-1}\). This avoids reusing the one-hop matrix \(W\) for multi-hop certified origins.

\section{Model Validation and Consensus}

\subsection{Consensus Process by Full Nodes}
\label{sec:virtual-voting}

Consensus is performed over origin tuples \(\omega\), not over node IDs or unconstrained model names. For notational simplicity, the epoch index is suppressed when clear. Let \(N_F=|\mathcal{F}|\). Define
\begin{equation}\label{eq:quorum_sizes}
    \tau_N=\left\lfloor\frac{2N}{3}\right\rfloor+1,
    \qquad
    \tau_F=\left\lfloor\frac{2N_F}{3}\right\rfloor+1.
\end{equation}
The control-plane model assumes authenticated channels, deterministic verification, fewer than \(N_F/3\) Byzantine full nodes, and at most \(\lfloor(N-1)/3\rfloor\) Byzantine nodes. Lazy nodes are learning-invalid but control-correct unless explicitly stated otherwise.

\subsubsection{Definitions}
For an event \(e\), let \(s(e)\) be its self-parent and \(\mathcal{P}\{e\}\) its gossip parent if one exists. Since the stored edge direction is parent-to-child, \(u\preceq e\) means \(u\) is an ancestor of \(e\), i.e., \(u\to^\star e\). At genesis, \(\mathcal{M}\{e_i(0)\}=1\). For an empty heartbeat event, \(\mathcal{M}\{e_i(r)\}=\mathcal{M}\{s(e_i(r))\}\), but the event is non-voting and is not included in a reachable voting set.\par

For a non-empty event \(e_j(r)\), let
\begin{equation}\label{eq:base_moment}
  M=
  \max\big(
  \mathcal{M}\{s(e_j(r))\},
  \mathcal{M}\{\mathcal{P}\{e_j(r)\}\}
  \big).
\end{equation}
The set \(\mathcal{R}\{e_j(r)\}\) contains the earliest reachable eligible events of moment \(M\), with at most one event per creator. An eligible event is a genesis event if \(M=1\), and a non-empty voting event if \(M>1\). If multiple reachable eligible events of moment \(M\) are created by the same node, the one with the smallest ripple index is retained.\par

A non-empty event is a voting event if
\begin{equation}\label{eq:voting_event_condition}
  |\mathcal{R}\{e_j(r)\}|\geq \tau_N .
\end{equation}
Its moment is
\begin{equation}\label{eq:M_mapping}
  \mathcal{M}\{e_j(r)\} =
  \begin{cases}
    M+1, & \text{if } e_j(r) \text{ is voting},\\
    M,   & \text{otherwise}.
  \end{cases}
\end{equation}

For a voting event $e_j(r)$, its virtual vote on origin tuple $\omega$
is defined in \eqref{eq:virtual_vote}, where $g(\omega)$ is the genesis
event associated with $\omega$. The intermediate confirmation vote is
defined in \eqref{eq:intermediate_confirmation}.

\begin{figure*}[!b]
\hrule
\vspace{0.6\baselineskip}

\begin{equation}
\label{eq:virtual_vote}
\begin{aligned}
\mathcal{V}(e_j(r),\omega)
=
\begin{cases}
1,
&
\mathcal{M}\{e_j(r)\}=2
\text{ and } g(\omega)\in\mathcal{R}\{e_j(r)\},
\\[2pt]
1,
&
\mathcal{M}\{e_j(r)\}>2
\text{ and }
\displaystyle
\sum_{e\in\mathcal{R}\{e_j(r)\}}
\mathcal{V}(e,\omega)
>
|\mathcal{R}\{e_j(r)\}|/2,
\\[3pt]
0,
&
\text{otherwise}.
\end{cases}
\end{aligned}
\end{equation}

\vspace{+0.8\baselineskip}

\begin{equation}
\label{eq:intermediate_confirmation}
\begin{aligned}
\mathcal{U}(e_j(r),\omega)
=
\begin{cases}
1,
&
\mathcal{M}\{e_j(r)\}>2
\text{ and }
\displaystyle
\sum_{e\in\mathcal{R}\{e_j(r)\}}
\mathcal{V}(e,\omega)
\geq \tau_N,
\\[2pt]
0,
&
\text{otherwise}.
\end{cases}
\end{aligned}
\end{equation}

\vspace{-0.4\baselineskip}
\end{figure*}

Let \(\mathcal{O}_j(r)\) be the set of origin tuples observed by node \(j\) up to ripple \(r\). Duplicates do not enlarge \(\mathcal{O}_j(r)\).

\subsubsection{Local confirmation and certificate exchange}
Let \(C(r)\) be the certified set by the end of ripple \(r\), with \(C(0)=\varnothing\). A full node locally confirms \(\omega\notin C(r-1)\) if there exists a voting event \(e_j(r)\) with
\[
    \mathcal{U}(e_j(r),\omega)=1.
\]
Full nodes exchange compact local-confirmation vectors. An origin tuple \(\omega\) enters \(C(r)\) if at least \(\tau_F\) full nodes locally confirm it. Node \(j\) is ready if
\begin{equation}\label{eq:ready_condition}
    |C(r)\cap \mathcal{O}_j(r)|\geq Q .
\end{equation}
If at least \(\tau_N\) nodes are ready, a full node signs a termination certificate for \((t,r,C(r))\). A light node terminates after receiving at least \(\tau_F\) matching certificates. This step is a compact full-node certificate exchange; virtual voting itself is the local inference of confirmation from the Topology DAG.

\subsection{Consensus Properties}
\label{subsec:consensus_properties}

\blue{
The next results concern provenance finality over origin tuples. They do not assert semantic validity, and they do not assert identical local model parameters after every epoch. Proofs are in Appendix~\ref{app:consensus_proofs}.
}

\blue{
\begin{lem}[Well-defined virtual voting]
\label{lem:well_defined_voting}
For any finite Topology DAG constructed from valid event certificates and accepted proofs, the quantities \(\mathcal{M}\{e\}\), \(\mathcal{R}\{e\}\), \(\mathcal{V}(e,\omega)\), and \(\mathcal{U}(e,\omega)\) are uniquely defined for every event \(e\) and every non-equivocated origin tuple \(\omega\).
\end{lem}
}

\blue{
\begin{theo}[Quorum-intersection safety]
\label{theo:quorum_safety}
Assume authenticated communication, deterministic verification, \(\tau_F=\lfloor2N_F/3\rfloor+1\), and fewer than \(N_F/3\) Byzantine full nodes. Then two control-correct nodes cannot accept two different termination certificates for the same epoch.
\end{theo}
}

\blue{
\begin{theo}[Virtual-voting consistency]
\label{theo:vv_consistency}
Consider two control-correct full nodes whose Topology DAG prefixes contain the same valid event-certificate and accepted-proof set up to ripple \(r\). Then their locally computed \(\mathcal{M}\), \(\mathcal{R}\), \(\mathcal{V}\), \(\mathcal{U}\), and local-confirmation vectors at ripple \(r\) are identical.
\end{theo}
}

\blue{
\begin{theo}[Conditional liveness]
\label{theo:conditional_liveness}
Assume eventual delivery of valid control messages, bounded control-correct processing delay, connected control-correct gossip paths, and fewer than \(N_F/3\) Byzantine full nodes. Suppose there exists a finite ripple \(r^\star\) at which at least \(\tau_N\) control-correct nodes have each observed at least \(Q\) non-equivocated origin tuples that pass payload validation and are confirmable from the valid Topology DAG. Then every control-correct node eventually receives a valid termination certificate and exits the gossip stage.
\end{theo}
}

\subsection{Robustness, Complexity, and Validation}
\label{other-validations}

Corrupted control messages are rejected because sender signatures, receiver endorsements, self-parent hashes, gossip-parent hashes, origin tuples, model hashes, epoch/ripple indices, and duplicate/equivocation rules must all verify. Lost data messages create heartbeat events and may increase the number of ripples. Lost control messages delay Topology-DAG reconstruction but cannot create false edges. Thus, loss and delay affect latency and liveness, while certificate safety follows from Theorem~\ref{theo:quorum_safety}.\par

In each ripple, each node sends at most one model payload over the data channel, giving \(O(N)\) local payload transmissions. Each node also sends one event certificate to full nodes, and accepted payloads add accepted proofs; this costs \(O(NN_F)\) control messages per ripple. Full-node confirmation exchange costs \(O(N_F^2N)\) bits per ripple. With cached reachability bitsets, moment computation and virtual voting cost \(O(N^2)\) bit operations per full node per ripple. The Topology DAG stores \(O(NR_t)\) event metadata per epoch.

\subsubsection{Payload validation}
\label{sssec:payload_validation}
If node \(i\) receives \(z_\omega\), it first computes
\begin{equation}\label{eq:f1}
     f_1(z_\omega)=\|z_\omega-\theta_i^{t-1}\|_2.
\end{equation}
If \(f_1(z_\omega)<\epsstale\), the update is rejected as stale before any directional normalization. Otherwise, using the previously post-audit set, node \(i\) computes \(\mu_i^1,\sigma_i^1\) and rejects if
\[
    f_1(z_\omega)<L_i^1
    \quad\text{or}\quad
    f_1(z_\omega)>U_i^1,
\]
where
\[
    L_i^1=\max(\epsstale,\mu_i^1-3\sigma_i^1),
    \qquad
    U_i^1=\mu_i^1+3\sigma_i^1.
\]

For direction, define the valid previous direction pool by excluding updates with displacement below \(\epsstale\). If this pool is too small or if the mean direction norm is below \(\epsdir\), the directional filter is skipped for the current epoch and only the magnitude rule is used. Otherwise,
\begin{equation}\label{eq:f2}
     f_2(z_\omega)
     =
     \frac{
     \langle z_\omega-\theta_i^{t-1},\bar d_i^{t-1}\rangle
     }{\|z_\omega-\theta_i^{t-1}\|_2},
\end{equation}
where
\[
    \bar d_i^{t-1}
    =
    \frac{d_i^{t-1}}{\|d_i^{t-1}\|_2},
\]
and \(d_i^{t-1}\) is the average normalized direction over the valid previous pool. Node \(i\) rejects if \(f_2(z_\omega)<L_i^2=\mu_i^2-3\sigma_i^2\). These filters are first-stage screens inspired by trust and asynchronous robust validation~\cite{FLTrust2021,ZenoPP2020}, not universal Byzantine detectors.

\blue{
\begin{propo}[False-rejection control under concentrated honest statistics]
\label{prop:filter_false_rejection}
Fix node \(i\) and statistic \(f_m\), \(m\in\{1,2\}\). Suppose learning-clean values of \(f_m\) are sub-Gaussian with mean \(\bar\mu_m\) and scale \(\bar\sigma_m\), and suppose the empirical estimates satisfy
\[
|\mu_i^m-\bar\mu_m|\leq \varepsilon_\mu,\qquad
|\sigma_i^m-\bar\sigma_m|\leq \varepsilon_\sigma
\]
with probability at least \(1-\delta_n\). Then the false-rejection probability is bounded by
\[
    \Pr\{\text{clean rejection}\}
    \leq
    \delta_n+
    c_m
    \exp\!\left(
    -
    \frac{(3\bar\sigma_m-3\varepsilon_\sigma-\varepsilon_\mu)_+^2}
    {2\bar\sigma_m^2}
    \right),
\]
with \(c_1=2\) for the two-sided magnitude test and \(c_2=1\) for the one-sided directional test.
\end{propo}
}

\subsubsection{Accepted-proof validation}
\label{sssec:sync_reply}
Each node has a signing key and a public verification key. Full nodes insert an event only if the event certificate \(\chi_i^t(r)\) is valid and its self-parent is consistent. A gossip-parent edge is inserted only if the accepted proof \(\Pi_{a\to i}^t(r)\) verifies both the sender proof and receiver endorsement. Duplicate event certificates from the same creator in the same epoch/ripple trigger an equivocation mark. Conflicting genesis tuples from the same origin and epoch also trigger an equivocation mark and are excluded from certification. Standard Ed25519 signatures can implement these checks~\cite{RFC8032,Ed25519Bernstein2012}.

\subsubsection{Semantic consistency audit}
\label{sssec:semantic_audit}
After control-plane finality, node \(i\) audits each locally stored \(z_\omega\in\mathcal{A}_i^t\) using a private trigger-free validation set \(\mathsf{D}_i\). Let \(o(x;z)\) denote the output vector of model \(z\). Node \(i\) computes
\begin{equation}\label{eq:semantic_distance}
    d_i(z_\omega)
    =
    \max_{x\in\mathsf{D}_i}
    \left\|
    o(x;z_\omega)-o(x;\theta_i^{t-1})
    \right\|_2 .
\end{equation}
It removes updates whose score exceeds \(\mathrm{median}+3\,\mathrm{MAD}\), following robust outlier practice~\cite{Iglewicz1993,Leys2013}. If \(|\mathcal{A}_i^t|<4\), pruning is skipped. If all models exceed the threshold, the model with the smallest \(d_i\) is retained. This audit assumes the previous local reference has not already drifted beyond the audit tolerance; under persistent reference poisoning, the audit becomes empirical.

\blue{
\begin{propo}[Semantic separation]
\label{prop:semantic_separation}
Let \(\mathcal{H}_i^t\) and \(\mathcal{B}_i^t\) be the learning-clean and Byzantine tuples in \(\mathcal{A}_i^t\). Suppose \(|\mathcal{B}_i^t|<|\mathcal{A}_i^t|/2\), and there exist \(a_i^t<b_i^t\) such that
\[
    d_i(z_\omega)\le a_i^t,\quad \omega\in\mathcal{H}_i^t,
    \qquad
    d_i(z_\omega)\ge b_i^t,\quad \omega\in\mathcal{B}_i^t.
\]
If \(b_i^t>\mathrm{med}_i^t+3\,\mathrm{MAD}_i^t\ge a_i^t\), then the MAD audit removes all Byzantine tuples in \(\mathcal{B}_i^t\) and retains all learning-clean tuples in \(\mathcal{H}_i^t\).
\end{propo}
}

\subsection{Learning Effect of Certified and Filtered Aggregation}
\label{subsec:learning_effect}

\blue{
Let \(H=|\mathcal{H}_{\mathrm{L}}|\) and \(\bar\theta^t=H^{-1}\sum_{i\in\mathcal{H}_{\mathrm{L}}}\theta_i^t\). Let \(W_t^{\mathrm{eff}}\) be the row-stochastic effective aggregation matrix induced by the certified observed set and the post-audit weights in \eqref{ours_agg}, restricted to learning-clean nodes and with invalid residual influence represented as a perturbation. We assume the effective mixing condition
\begin{equation}\label{eq:effective_mixing}
    \EX\!\left[
    \left\|
    P W_t^{\mathrm{eff}}
    \right\|_2^2
    \middle|\mathfrak{F}_t
    \right]
    \le \rho^2<1,
\end{equation}
where \(P=I-H^{-1}\bm{1}\bm{1}^{\top}\) and \(\mathfrak{F}_t\) is the training-history filtration before aggregation at epoch \(t\). This assumption replaces the fixed-matrix condition in ordinary DFL and matches the filtered, time-varying aggregation used by gspDAG-FL.
}

\blue{
We also assume the primal--dual tracking error is bounded:
\begin{equation}
\label{eq:tracking_assumption}
\begin{aligned}
&\frac{1}{HK}
\sum_{i\in\mathcal{H}_{\mathrm{L}}}
\sum_{k=0}^{K-1}
\EX\!\left[
\left\|
\EX[g_i^{t,k}\mid\mathfrak{F}_t]
-
\nabla F_i(\theta_i^{t,k})
\right\|_2^2
\right]
\\
&\hspace{38mm}
\le
c_\lambda(\Omega_t+\delta_t^2),
\end{aligned}
\end{equation}
where \(g_i^{t,k}\) is the stochastic gradient including the primal--dual correction, \(\Omega_t\) is the learning-clean disagreement energy defined in Appendix~\ref{app:validation_learning_proofs}, and \(c_\lambda<\infty\).
Let \(\xi_i^t=\eta K b_i^t\) be the parameter-space aggregation perturbation, with
\[
    \frac{1}{H}\sum_{i\in\mathcal{H}_{\mathrm{L}}}
    \EX\|b_i^t\|_2^2
    \le \delta_t^2.
\]
}

\blue{
\begin{theo}[Stationarity under certified time-varying gossip]
\label{theo:learning_convergence}
Assume learning-clean losses are \(L\)-smooth and lower bounded, stochastic gradients have variance at most \(\sigma_g^2\), heterogeneity is bounded by \(\zeta^2\), \eqref{eq:effective_mixing} holds, and the tracking condition \eqref{eq:tracking_assumption} holds. For sufficiently small \(\eta\),
\begin{align}
\frac{1}{T}
\sum_{t=0}^{T-1}
\EX\!\left[
\left\|
\nabla F_{\mathcal{H}_{\mathrm{L}}}(\bar\theta^t)
\right\|_2^2
\right]
&\le
O\!\left(
\frac{
F_{\mathcal{H}_{\mathrm{L}}}(\bar\theta^0)-F_{\inf}
}{
\eta KT
}
\right)
+
O(\eta\sigma_g^2)
\notag
\\
&\quad+
O\!\left(
\frac{\eta K\zeta^2}{(1-\rho)^2}
\right)
\notag
\\
&\quad+
O\!\left(
\frac{1}{T}
\sum_{t=0}^{T-1}
\delta_t^2
\right).
\end{align}
Thus, with \(\eta=\Theta(T^{-1/2})\) and bounded average perturbation, gspDAG-FL converges to a stationary neighborhood of \(F_{\mathcal{H}_{\mathrm{L}}}\).
\end{theo}
}

\section{Simulation Results}
\label{sec:simulations}
\begingroup
\color{blue}

This section evaluates gspDAG-FL in terms of learning quality, robustness, detection effectiveness, ripple dynamics, and ledger scalability. The experiments test whether gspDAG-FL preserves model quality under invalid participation, whether the validation stages remove different attack types, and whether the DAG control plane scales better than block-centric coordination.

\subsection{Simulation Setup}
\label{sec:setup}

\subsubsection{Experimental environment and baselines}
For gspDAG-FL, the ledger layer is implemented by extending the \texttt{dagsim} Hashgraph simulator~\cite{dagsim_ref} in \texttt{Kotlin} with JDK~17+ and \texttt{Gradle}. We augment events with origin tuples, payload hashes, event certificates, and accepted proofs. Full-node Topology-DAG processing uses \texttt{JGraphT}~\cite{jgrapht_ref}, and signature verification uses Ed25519 primitives from \texttt{BouncyCastle}~\cite{bouncycastle_ref}. Local learning and validation use \texttt{Python}~3.10, \texttt{PyTorch}~\cite{pytorch_ref}, and \texttt{NumPy}~\cite{numpy_ref}. Data and control channels are bidirectional \texttt{gRPC} streams using \texttt{gRPC-Kotlin}, \texttt{grpcio}, and \texttt{Protocol Buffers}~\cite{grpckotlin_ref,grpcio_ref,protobuf_ref}. Experiments are orchestrated with \texttt{Docker Compose}, logged to CSV, and visualized using \texttt{matplotlib}~\cite{matplotlib_ref}.\par

Large parameter sweeps and ripple-process stress tests use a trace-driven simulator with the same event transition, proof verification, duplicate, equivocation, and certification rules. Random seeds are fixed in each script. AD-PSGD uses the public \texttt{stochastic\_gradient\_push} implementation~\cite{adpsgd_repo}. BLADE-FL uses its public implementation~\cite{blade_repo} over a local Ganache Ethereum testnet~\cite{ganache_ref}. ChainFL uses the public ChainsFL implementation~\cite{gggg1}. Table~\ref{tab:baseline_setup} summarizes the adaptation of baselines.

\begin{table}[!t]
\caption{Baseline adaptation.}
\label{tab:baseline_setup}
\centering
\scriptsize
\setlength{\tabcolsep}{2.5pt}
\renewcommand{\arraystretch}{1.06}
\begin{adjustbox}{max width=\columnwidth}
\begin{tabular}{l|l|l|l}
\toprule
\textbf{Method} & \textbf{Topology} & \textbf{Finality} & \textbf{Validation} \\
\midrule
AD-PSGD & same gossip graph & none & none/local averaging \\
BLADE-FL & ledger broadcast & PoW block & public held-out set \\
ChainFL & shard/main ledger & shard+DAG commit & public held-out set \\
gspDAG-FL & local gossip & origin certificate & local/private audit \\
\bottomrule
\end{tabular}
\end{adjustbox}
\end{table}

\subsubsection{Data, tasks, and fairness}
Task~1 is MNIST image classification~\cite{mnist_ref}, using a lightweight CNN based on MobileNetV2~\cite{gggg12}; performance is clean test accuracy. Semantic attacks use a \(3\times3\) white square trigger and target class \(0\). Task~2 is Penn Treebank language modeling~\cite{penn_treebank_ref}, using a GRU next-word predictor; performance is clean test perplexity. Semantic attacks insert a fixed rare two-word trigger and force a target token.\par

All methods use the same effective training budget. The training pool is common to all methods. Public validation data for BLADE-FL and ChainFL and private audit data for gspDAG-FL are drawn from held-out samples and are not used for training. AD-PSGD does not use validation data. Final metrics use the same disjoint test split.

\subsubsection{Fault model, topology, and metrics}
Unless otherwise stated, the default setting is \(N=15\), target adversarial ratio \(\mu=0.15\), target lazy ratio \(\gamma=0.10\), full-node ratio \(0.40\), \(Q=\lfloor2N/3\rfloor+1\), and \(R_t^{\max}=Q+2\). Integer assignment uses
\[
    B=\min\{\lfloor\mu N\rfloor,\lfloor(N-1)/3\rfloor\},
    \qquad
    Z=\lfloor\gamma N\rfloor,
\]
with disjoint Byzantine and lazy sets. Lazy nodes are control-correct unless explicitly marked otherwise. In high-fault sweeps, the labels \(\mu,\gamma\) denote target ratios after this integer assignment. For \(N=5\), the \(0.40\) full-node ratio gives \(N_F=2\), so the full-node Byzantine bound permits no Byzantine full node; the simulation enforces this by sampling full nodes from the control-correct set.\par

For gspDAG-FL and AD-PSGD, the communication graph is a connected Watts--Strogatz small-world graph generated with NetworkX~\cite{networkx_ref}. The degree is \(k=8\) for \(N\le20\), \(k=10\) for \(21\le N\le30\), and \(k=12\) for \(N\ge50\). The rewiring probability is \(p=0.25\) for \(N<10\), \(p=0.20\) for \(10\le N\le30\), and \(p=0.18\) for larger networks. BLADE-FL and ChainFL use their native complete or shard-level communication structures.\par

Local learning uses \(B=50\), \(K=5\), \(\eta=0.01\), and \(\lambda=100\). Node bandwidths are sampled between \(10\) and \(50\) Mbps, and latencies between \(50\) and \(200\) ms. Detection trials use \(20\) independent runs. Convergence is declared when
\begin{equation}\label{eq:sim_stopping}
   \frac{|\mathcal{L}^t-\mathcal{L}^{t-1}|}{\mathcal{L}^{t-1}}<10^{-3}
\end{equation}
for five consecutive epochs, where \(\mathcal{L}^t\) is the trimmed mean of node-local losses after removing the highest and lowest \(10\%\). Latency is time for update exchange plus finality. Throughput is the number of correctly validated learning-clean updates included per consensus round.

\begin{table}[!t]
\caption{Simulation configuration.}
\label{tab:simulation_setup}
\centering
\scriptsize
\setlength{\tabcolsep}{3.0pt}
\renewcommand{\arraystretch}{1.04}
\begin{adjustbox}{max width=\columnwidth}
\begin{tabular}{l|l}
\toprule
\textbf{Parameter} & \textbf{Default / sweep value} \\
\midrule
Tasks & MNIST; Penn Treebank \\
Baselines & AD-PSGD, BLADE-FL, ChainFL \\
Default nodes & \(N=15\) \\
Network-size sweep & \(5,10,15,20,25,30,50,75,100\) \\
Target Byzantine ratio & \(\mu=0.15\) default \\
Target lazy ratio & \(\gamma=0.10\) default \\
Full-node ratio & \(0.40\) \\
Readiness threshold & \(Q=\lfloor2N/3\rfloor+1\) \\
Ripple cap & \(R_t^{\max}=Q+2\) \\
Local SGD & \(B=50,\ K=5,\ \eta=0.01,\ \lambda=100\) \\
Network heterogeneity & \(10\)--\(50\) Mbps, \(50\)--\(200\) ms \\
Invalid-origin types & magnitude, directional, semantic, lazy replay \\
\bottomrule
\end{tabular}
\end{adjustbox}
\end{table}

\subsection{Results and Analysis}
\label{sec:results}

\subsubsection{Learning behavior and robustness}
Fig.~\ref{fig:loss_curves} shows representative training-loss trajectories of gspDAG-FL. Higher Byzantine participation raises the final loss floor because Byzantine updates perturb direction or semantics. Higher lazy participation mainly slows convergence because stale replays reduce origin freshness but do not necessarily push the model in a malicious direction.

\begin{figure*}[!t]
    \centering
    \begin{subfigure}[b]{0.485\textwidth}
        \centering
        \includegraphics[width=\linewidth]{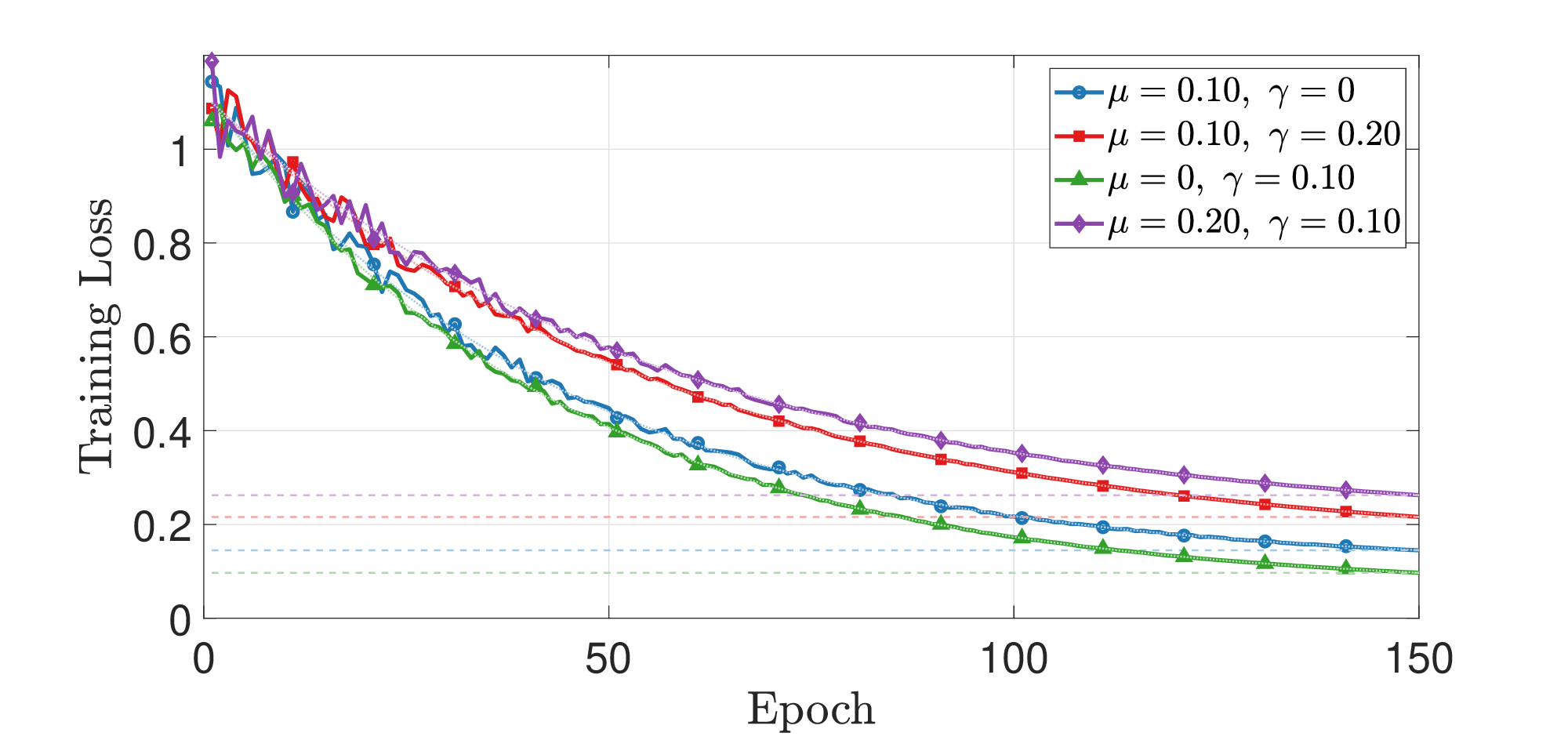}
        \caption{Task~1.}
        \label{fig:mnist_loss}
    \end{subfigure}
    \hfill
    \begin{subfigure}[b]{0.485\textwidth}
        \centering
        \includegraphics[width=\linewidth]{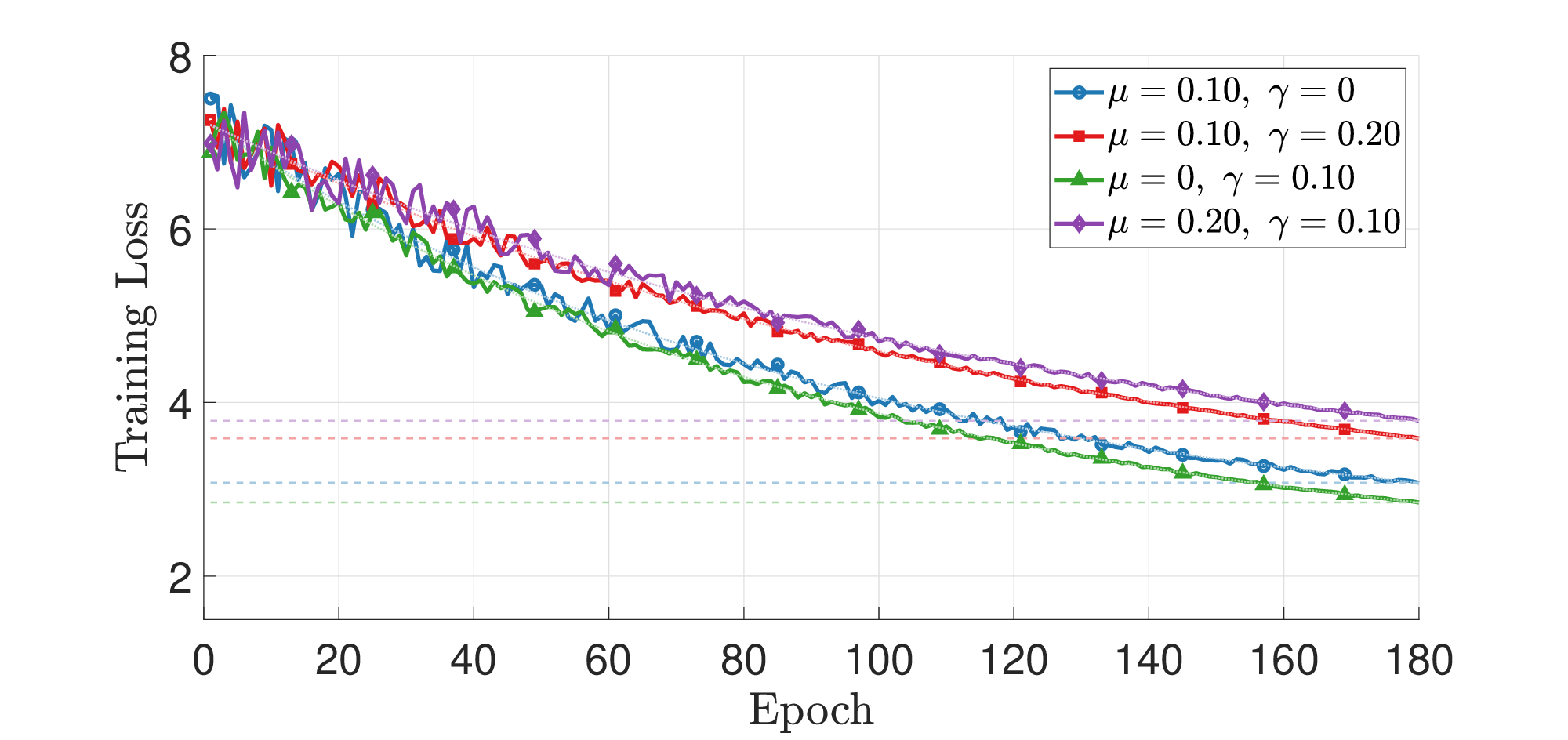}
        \caption{Task~2.}
        \label{fig:ptb_loss}
    \end{subfigure}
    \caption{Training-loss dynamics of gspDAG-FL under different Byzantine and lazy-node compositions.}
    \label{fig:loss_curves}
\end{figure*}

Table~\ref{tab:compact_learning_vs_N} reports learning performance as \(N\) increases. gspDAG-FL stays close to ChainFL and BLADE-FL because all three schemes include update-admission mechanisms. AD-PSGD is weaker because it has no finality layer; once a poisoned or stale update is locally averaged, later gossip can propagate its effect. At larger \(N\), gspDAG-FL benefits from a larger pool of learning-clean origins while the Topology DAG prevents most invalid origins from becoming control-plane certified.

\begin{table*}[!t]
\caption{Learning performance versus network size under fixed target \(\mu=0.15\) and \(\gamma=0.10\). Best values are bolded.}
\label{tab:compact_learning_vs_N}
\centering
\scriptsize
\setlength{\tabcolsep}{3.0pt}
\renewcommand{\arraystretch}{1.05}
\begin{adjustbox}{max width=\textwidth}
\begin{tabular}{ll|ccccccccc}
\toprule
\textbf{Task} & \textbf{Method}
& \multicolumn{9}{c}{\textbf{Number of nodes} \(N\)} \\
\cmidrule(lr){3-11}
& & 5 & 10 & 15 & 20 & 25 & 30 & 50 & 75 & 100 \\
\midrule
\multirow{4}{*}{Task 1 Acc. \(\uparrow\)}
& gspDAG-FL & 0.935 & \best{0.958} & 0.972 & \best{0.976} & 0.979 & \best{0.981} & \best{0.984} & \best{0.986} & \best{0.987} \\
& AD-PSGD   & 0.900 & 0.915 & 0.930 & 0.940 & 0.945 & 0.950 & 0.954 & 0.957 & 0.959 \\
& BLADE-FL  & 0.932 & 0.954 & 0.969 & 0.973 & 0.977 & 0.978 & 0.980 & 0.982 & 0.983 \\
& ChainFL   & \best{0.936} & 0.957 & \best{0.973} & 0.975 & \best{0.980} & 0.980 & 0.983 & 0.984 & 0.985 \\
\midrule
\multirow{4}{*}{Task 2 PPL. \(\downarrow\)}
& gspDAG-FL & \best{152.0} & 140.0 & 134.0 & \best{131.0} & \best{128.0} & \best{126.0} & \best{121.4} & \best{117.6} & \best{115.3} \\
& AD-PSGD   & 180.0 & 170.0 & 164.0 & 160.0 & 158.0 & 156.0 & 154.2 & 152.8 & 151.6 \\
& BLADE-FL  & 156.0 & 143.0 & 137.0 & 135.0 & 131.0 & 129.0 & 124.7 & 122.3 & 120.8 \\
& ChainFL   & 153.0 & \best{139.0} & \best{133.0} & 132.0 & 129.0 & 127.0 & 122.6 & 119.4 & 118.1 \\
\bottomrule
\end{tabular}
\end{adjustbox}
\end{table*}

Table~\ref{tab:compact_robustness} separates adversarial and lazy effects. Increasing \(\mu\) causes sharper degradation than increasing \(\gamma\), especially for Task~2. BLADE-FL and ChainFL remain robust, but their validation is tied to public-set admission and heavier ledger coordination. gspDAG-FL retains comparable learning quality while preserving local payload paths.

\begin{table*}[!t]
\caption{Robustness under adversarial and lazy participation at \(N=15\). Best values are bolded.}
\label{tab:compact_robustness}
\centering
\scriptsize
\setlength{\tabcolsep}{2.65pt}
\renewcommand{\arraystretch}{1.04}
\begin{adjustbox}{max width=\textwidth}
\begin{tabular}{ll|cccccc|ccccccc}
\toprule
\multirow{2}{*}{\textbf{Task}} & \multirow{2}{*}{\textbf{Method}}
& \multicolumn{6}{c|}{\textbf{Target adversarial ratio} \(\mu\), fixed \(\gamma=0.10\)}
& \multicolumn{7}{c}{\textbf{Target lazy ratio} \(\gamma\), fixed \(\mu=0.15\)} \\
\cmidrule(lr){3-8}\cmidrule(lr){9-15}
& & 0.05 & 0.10 & 0.15 & 0.20 & 0.25 & 0.30
& 0 & 0.05 & 0.10 & 0.15 & 0.20 & 0.25 & 0.30 \\
\midrule
\multirow{4}{*}{Task 1 Acc. \(\uparrow\)}
& gspDAG-FL & \best{0.982} & 0.977 & \best{0.972} & 0.966 & 0.958 & \best{0.946}
& \best{0.975} & \best{0.973} & \best{0.972} & \best{0.970} & \best{0.968} & \best{0.966} & \best{0.963} \\
& AD-PSGD & 0.945 & 0.938 & 0.930 & 0.918 & 0.900 & 0.870
& 0.938 & 0.934 & 0.930 & 0.925 & 0.920 & 0.915 & 0.910 \\
& BLADE-FL & 0.980 & 0.976 & 0.971 & 0.965 & 0.957 & 0.944
& 0.973 & 0.971 & 0.969 & 0.967 & 0.964 & 0.962 & 0.959 \\
& ChainFL & 0.981 & \best{0.978} & 0.971 & \best{0.967} & \best{0.959} & 0.945
& 0.974 & 0.972 & 0.971 & 0.969 & 0.967 & 0.965 & 0.962 \\
\midrule
\multirow{4}{*}{Task 2 PPL. \(\downarrow\)}
& gspDAG-FL & \best{128.0} & 131.0 & \best{134.0} & 140.0 & \best{148.0} & 160.0
& \best{132.0} & 133.0 & \best{134.0} & \best{135.0} & \best{136.0} & \best{138.0} & \best{139.0} \\
& AD-PSGD & 150.0 & 157.0 & 164.0 & 176.0 & 192.0 & 215.0
& 160.0 & 162.0 & 164.0 & 167.0 & 170.0 & 173.0 & 176.0 \\
& BLADE-FL & 130.0 & 132.0 & 136.0 & 141.0 & 150.0 & 162.0
& 134.0 & 135.0 & 136.0 & 138.0 & 140.0 & 142.0 & 144.0 \\
& ChainFL & 129.0 & \best{130.0} & 135.0 & \best{139.0} & 149.0 & \best{159.0}
& 133.0 & \best{133.0} & \best{134.0} & 136.0 & 137.0 & 139.0 & 140.0 \\
\bottomrule
\end{tabular}
\end{adjustbox}
\end{table*}

\subsubsection{Validation pipeline and ripple dynamics}
Table~\ref{tab:defense_pipeline_dense} reports end-to-end detection rates and false alarms. The false-alarm rate stays below \(0.4\%\), so robustness is not obtained by discarding many learning-clean updates. Table~\ref{tab:cumulative_detection_dense} shows how stages contribute over time: magnitude filtering dominates early, directional detections rise after the direction reference stabilizes, and semantic detections become more important after finality.

\begin{table}[!t]
\caption{End-to-end defense-pipeline summary under \(N=15\), target \(\mu=0.15\), and target \(\gamma=0.10\). Counts are mean \(\pm\) standard deviation over \(20\) runs.}
\label{tab:defense_pipeline_dense}
\centering
\scriptsize
\setlength{\tabcolsep}{2.6pt}
\renewcommand{\arraystretch}{1.05}
\begin{adjustbox}{max width=\columnwidth}
\begin{tabular}{l|ccc}
\toprule
\textbf{Task} & \textbf{Detection rate} & \textbf{False alarm} & \textbf{Detected / flawed} \\
\midrule
Task~1 & \(96.1\%\) & \(0.35\%\) & \(432.6\pm13.3\;/\;450\) \\
Task~2 & \(95.7\%\) & \(0.34\%\) & \(502.6\pm14.6\;/\;525\) \\
\bottomrule
\end{tabular}
\end{adjustbox}
\end{table}

\begin{table*}[!t]
\caption{Cumulative invalid-origin detection by validation stage.}
\label{tab:cumulative_detection_dense}
\centering
\scriptsize
\setlength{\tabcolsep}{3.0pt}
\renewcommand{\arraystretch}{1.05}
\begin{adjustbox}{max width=\textwidth}
\begin{tabular}{ll|ccccc|ccccc}
\toprule
\multirow{2}{*}{\textbf{Quantity}} & \multirow{2}{*}{\textbf{Stage}}
& \multicolumn{5}{c|}{\textbf{Task 1 epochs}}
& \multicolumn{5}{c}{\textbf{Task 2 epochs}} \\
\cmidrule(lr){3-7}\cmidrule(lr){8-12}
& & 20 & 40 & 60 & 80 & 100 & 24 & 48 & 72 & 96 & 120 \\
\midrule
\multirow{5}{*}{Cumulative count}
& Magnitude & 38 & 72 & 101 & 128 & 150 & 50 & 94 & 134 & 164 & 180 \\
& Direction & 16 & 34 & 54 & 72 & 88 & 18 & 40 & 62 & 86 & 104 \\
& Semantic  & 3  & 8  & 17 & 30 & 50 & 1  & 4  & 12 & 24 & 61 \\
& Undetected & 3 & 6 & 8 & 10 & 12 & 3 & 6 & 8 & 14 & 15 \\
& Detected total & 57 & 114 & 172 & 230 & 288 & 69 & 138 & 208 & 274 & 345 \\
\midrule
\multicolumn{2}{l|}{Invalid origins}
& 60 & 120 & 180 & 240 & 300 & 72 & 144 & 216 & 288 & 360 \\
\bottomrule
\end{tabular}
\end{adjustbox}
\end{table*}

Fig.~\ref{fig:ripple_dynamics} shows the per-epoch ripple count \(R_t\). Clean regimes stay close to the useful receive-opportunity lower bound \(Q-1\). Adversary-heavy regimes need more ripples because invalid origins are rejected before contributing to readiness. Lazy-heavy regimes also increase \(R_t\), but less sharply, because control-correct forwarders can still carry fresh origins.

\begin{figure*}[!t]
    \centering
    \begin{subfigure}[b]{0.485\textwidth}
        \centering
        \includegraphics[width=\linewidth]{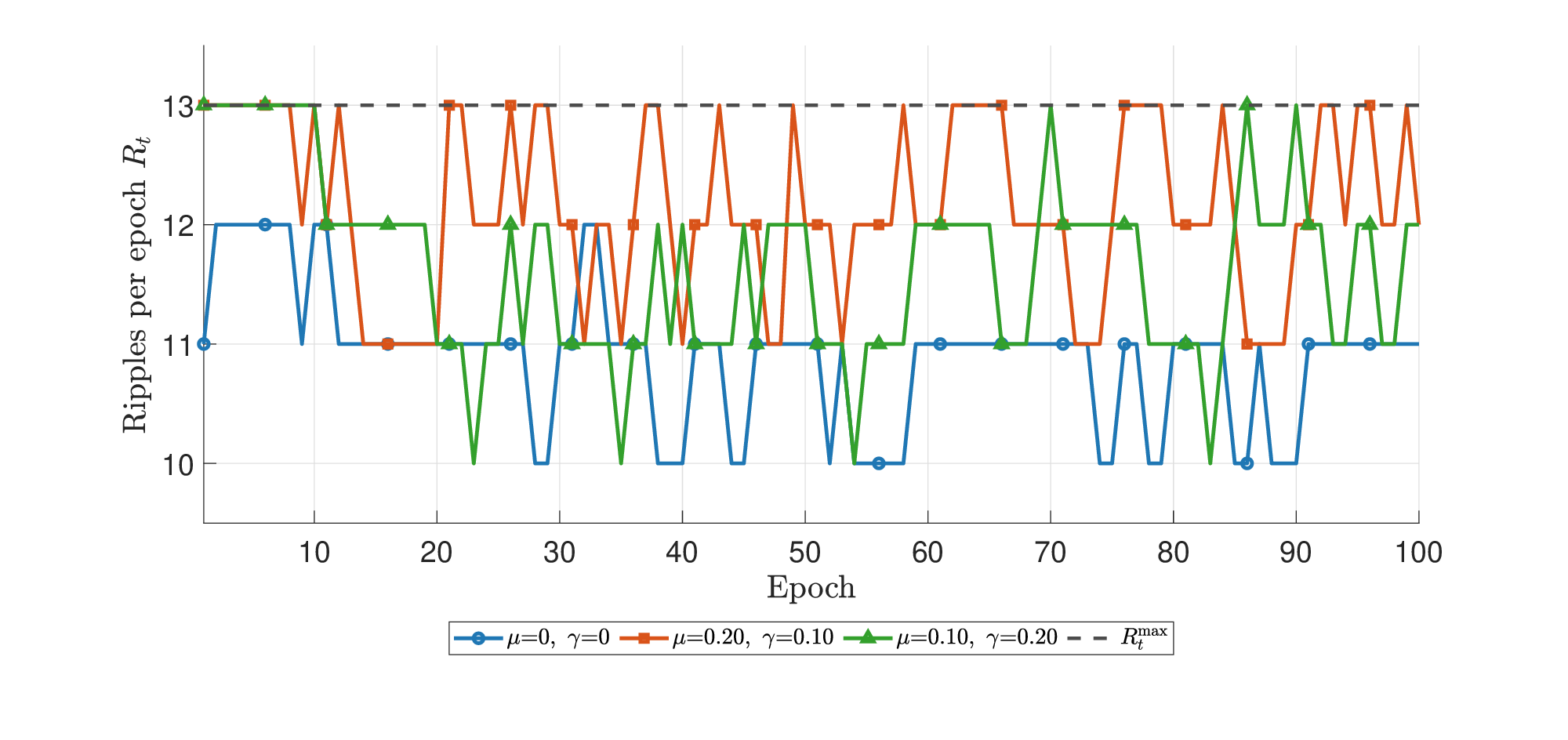}
        \caption{Task~1.}
        \label{fig:rt_task1}
    \end{subfigure}
    \hfill
    \begin{subfigure}[b]{0.485\textwidth}
        \centering
        \includegraphics[width=\linewidth]{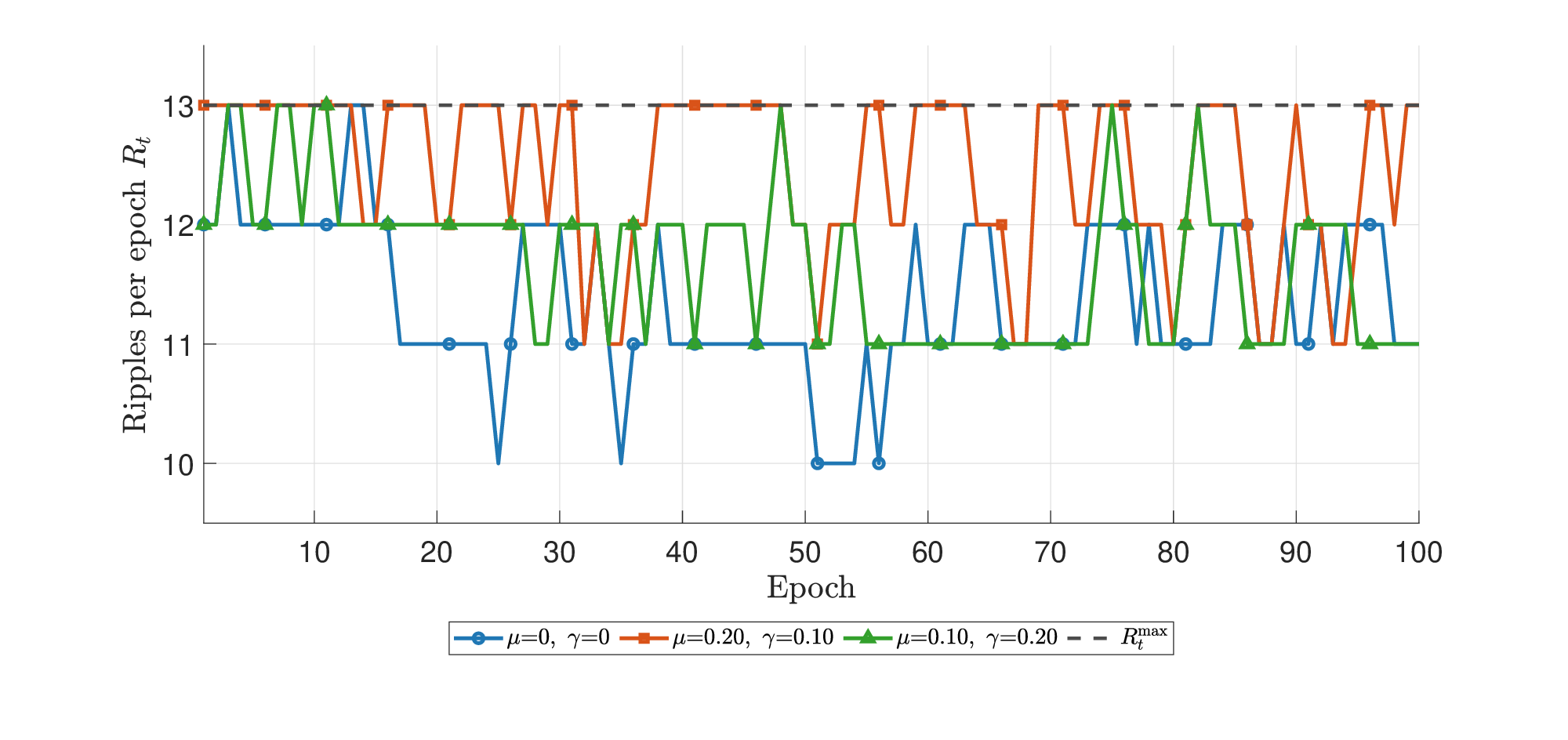}
        \caption{Task~2.}
        \label{fig:rt_task2}
    \end{subfigure}
    \caption{Ripple dynamics of gspDAG-FL for \(N=15\), \(Q=11\), and \(R_t^{\max}=13\).}
    \label{fig:ripple_dynamics}
\end{figure*}

\begin{table}[!t]
\caption{Ripple statistics for \(N=15\), \(Q=11\), and \(R_t^{\max}=13\).}
\label{tab:ripple_statistics_dense}
\centering
\scriptsize
\setlength{\tabcolsep}{2.8pt}
\renewcommand{\arraystretch}{1.05}
\begin{adjustbox}{max width=\columnwidth}
\begin{tabular}{l|ccc|ccc}
\toprule
\multirow{2}{*}{\textbf{Regime}}
& \multicolumn{3}{c|}{\textbf{Task 1}}
& \multicolumn{3}{c}{\textbf{Task 2}} \\
\cmidrule(lr){2-4}\cmidrule(lr){5-7}
& med. & mean & cap
& med. & mean & cap \\
\midrule
\(\mu=0,\gamma=0\)
& \best{11} & \best{10.84} & \best{3.7\%}
& \best{12} & \best{11.13} & \best{8.6\%} \\
\(\mu=0.20,\gamma=0.10\)
& 12 & 11.91 & 12.4\%
& 13 & 12.28 & 24.7\% \\
\(\mu=0.10,\gamma=0.20\)
& 11 & 11.27 & 7.3\%
& 12 & 11.62 & 13.9\% \\
\bottomrule
\end{tabular}
\end{adjustbox}
\end{table}

\subsubsection{System scaling}
Table~\ref{tab:compact_system_scaling} reports normalized latency, throughput, and convergence rounds. Each value is normalized by the same method's value at \(N=5\), so the table reports scaling trend rather than absolute hardware-specific latency. gspDAG-FL keeps model tensors on local gossip paths and sends only metadata through the full-node control plane. BLADE-FL globally propagates transactions and waits for block confirmation. ChainFL reduces the bottleneck through sharding, but shard consensus and main-chain commitment still add synchronization overhead.

\begin{table*}[!t]
\caption{Normalized system scaling under fixed target \(\mu=0.15\) and \(\gamma=0.10\). Values are normalized to each method's value at \(N=5\). Best values are bolded.}
\label{tab:compact_system_scaling}
\centering
\scriptsize
\setlength{\tabcolsep}{2.7pt}
\renewcommand{\arraystretch}{1.03}
\begin{adjustbox}{max width=\textwidth}
\begin{tabular}{lll|ccccccccc}
\toprule
\textbf{Task} & \textbf{Metric} & \textbf{Method}
& \multicolumn{9}{c}{\textbf{Number of nodes} \(N\)} \\
\cmidrule(lr){4-12}
& & & 5 & 10 & 15 & 20 & 25 & 30 & 50 & 75 & 100 \\
\midrule
\multirow{9}{*}{Task 1}
& \multirow{3}{*}{Lat. \(\downarrow\)}
& gspDAG-FL & \best{1.000} & \best{1.194} & \best{1.297} & \best{1.456} & \best{1.538} & \best{1.624} & \best{2.071} & \best{2.637} & \best{3.118} \\
& & ChainFL   & 1.000 & 1.319 & 1.642 & 1.836 & 2.027 & 2.476 & 3.691 & 5.183 & 6.742 \\
& & BLADE-FL  & 1.000 & 1.397 & 1.743 & 2.108 & 2.493 & 2.897 & 4.582 & 6.914 & 9.227 \\
\cmidrule(lr){2-12}
& \multirow{3}{*}{Thr. \(\uparrow\)}
& gspDAG-FL & \best{1.000} & \best{1.704} & \best{2.427} & \best{3.012} & \best{3.449} & \best{4.013} & \best{5.674} & \best{7.386} & \best{8.832} \\
& & ChainFL   & 1.000 & 1.486 & 1.932 & 2.286 & 2.617 & 2.803 & 3.548 & 4.213 & 4.762 \\
& & BLADE-FL  & 1.000 & 0.982 & 0.921 & 0.832 & 0.731 & 0.653 & 0.514 & 0.407 & 0.342 \\
\cmidrule(lr){2-12}
& \multirow{3}{*}{Rounds \(\downarrow\)}
& gspDAG-FL & \best{1.000} & \best{0.887} & \best{0.826} & \best{0.798} & \best{0.731} & \best{0.648} & \best{0.557} & \best{0.509} & \best{0.476} \\
& & ChainFL   & 1.000 & 0.981 & 0.969 & 0.958 & 0.979 & 0.991 & 1.048 & 1.116 & 1.184 \\
& & BLADE-FL  & 1.000 & 1.071 & 1.163 & 1.257 & 1.382 & 1.553 & 2.036 & 2.571 & 2.924 \\
\midrule
\multirow{9}{*}{Task 2}
& \multirow{3}{*}{Lat. \(\downarrow\)}
& gspDAG-FL & \best{1.000} & \best{1.263} & \best{1.428} & \best{1.639} & \best{1.774} & \best{1.905} & \best{2.413} & \best{3.021} & \best{3.587} \\
& & ChainFL   & 1.000 & 1.492 & 2.083 & 2.314 & 2.579 & 2.981 & 4.381 & 6.128 & 7.943 \\
& & BLADE-FL  & 1.000 & 1.512 & 2.119 & 2.553 & 3.107 & 3.654 & 5.786 & 8.971 & 12.236 \\
\cmidrule(lr){2-12}
& \multirow{3}{*}{Thr. \(\uparrow\)}
& gspDAG-FL & \best{1.000} & \best{1.518} & \best{2.083} & \best{2.571} & \best{2.962} & \best{3.301} & \best{4.764} & \best{6.137} & \best{7.124} \\
& & ChainFL   & 1.000 & 1.331 & 1.527 & 2.004 & 2.118 & 2.347 & 3.093 & 3.671 & 4.138 \\
& & BLADE-FL  & 1.000 & 0.987 & 0.869 & 0.801 & 0.648 & 0.569 & 0.447 & 0.353 & 0.284 \\
\cmidrule(lr){2-12}
& \multirow{3}{*}{Rounds \(\downarrow\)}
& gspDAG-FL & \best{1.000} & \best{0.941} & \best{0.898} & \best{0.869} & \best{0.842} & \best{0.819} & \best{0.771} & \best{0.731} & \best{0.704} \\
& & ChainFL   & 1.000 & 1.032 & 1.018 & 1.041 & 1.073 & 1.103 & 1.182 & 1.268 & 1.354 \\
& & BLADE-FL  & 1.000 & 1.052 & 1.129 & 1.231 & 1.421 & 1.704 & 2.247 & 2.768 & 3.164 \\
\bottomrule
\end{tabular}
\end{adjustbox}
\end{table*}

The throughput trend is the clearest system-level effect. Increasing \(N\) increases the number of local model origins available per epoch. In gspDAG-FL, these origins do not become all-to-all tensor broadcasts; only their signed provenance enters the control plane. Hence, more participants increase useful update diversity without creating the same confirmation bottleneck as block-centric ledgers.\par

The results support four conclusions. First, gspDAG-FL achieves learning quality comparable to validation-based ledger FL while preserving gossip-level model exchange. Second, Byzantine faults are more damaging than lazy faults because they alter update direction and semantics, not only freshness. Third, the validation stages are complementary. Fourth, the Topology DAG gives global provenance finality over origin tuples without forcing all nodes to hold identical model tensors in every epoch.

\endgroup

\section{Conclusion}
\label{sec:conclusion}
\begingroup
\color{blue}

This paper introduced gspDAG-FL, a secure decentralized federated learning framework that derives finality from gossip history rather than from a separate block, shard, or committee layer. Model payloads remain on local peer-to-peer paths, while full nodes use event certificates and receiver-endorsed accepted gossip proofs to reconstruct a Topology DAG and run virtual voting over provenance-admissible origin tuples. This gives global finality over which origins may be considered for aggregation without requiring identical local model states at all nodes.\par

The framework combines payload validation, accepted-proof validation, and private semantic audit to limit stale, malformed, and behaviorally abnormal updates. We proved well-definedness, quorum-intersection safety, virtual-voting consistency, conditional liveness, and a convergence guarantee under time-varying certified aggregation. Simulations on image classification and language modeling show that gspDAG-FL preserves learning quality close to validation-based ledger FL while improving latency, throughput, and scalability under Byzantine and lazy participation in the tested range up to \(N=100\).\par

Future work will study dynamic churn, adaptive full-node selection, stronger privacy for proof and audit metadata, and deployment under real edge-network traces.

\endgroup

\appendices
\begingroup
\color{blue}

\section{Proofs of Consensus Properties}
\label{app:consensus_proofs}

For an event \(e\), let \(\mathrm{cr}(e)\) denote its creator and \(\mathrm{rip}(e)\) its ripple. With the parent-to-child edge convention, \(u\preceq e\) means \(u\to^\star e\). Thus \(u\) is an ancestor visible from \(e\) by following parent links backward. Since all stored edges increase ripple index, the DAG is acyclic. Validity of event certificates and accepted proofs includes all signatures, endorsements, hashes, identities, origin tuples, and duplicate/equivocation checks.

\subsection{Proof of Lemma~\ref{lem:well_defined_voting}}

At ripple \(0\), all genesis moments and reachable sets are fixed. Assume all quantities are uniquely defined up to ripple \(r-1\). For any event \(e_j(r)\), its self-parent and possible gossip parent lie in earlier ripple \(r-1\), so their moments are unique. For a non-empty event, the base moment \(M\) in \eqref{eq:base_moment} is therefore unique. The ancestor relation \(\preceq\) is determined by the verified DAG. For each creator \(q\), the set of reachable eligible events of moment \(M\) is finite; if it is nonempty, the event with smallest ripple index is unique because a control-correct node creates at most one valid event certificate per ripple and equivocations are excluded. Hence the earliest-per-creator rule uniquely defines \(\mathcal{R}\{e_j(r)\}\). The voting test and moment update are deterministic, so \(\mathcal{M}\{e_j(r)\}\) is unique. For an empty heartbeat event, the inherited moment from the self-parent is unique and the event is non-voting.\par

Virtual votes are defined by induction over moment. Moment-2 votes are determined directly by whether \(g(\omega)\in\mathcal{R}\{e\}\). For moment \(m>2\), the vote of a voting event depends only on reachable eligible events of moment \(m-1\), whose votes have already been uniquely defined. The intermediate vote is a deterministic threshold function of these lower-moment votes. Therefore \(\mathcal{V}\) and \(\mathcal{U}\) are uniquely defined for every voting event and non-equivocated origin tuple. This proves the lemma.

\subsection{Proof of Theorem~\ref{theo:quorum_safety}}

Let \(B_F<N_F/3\) be the number of Byzantine full nodes. Any accepted certificate has at least
\[
    \tau_F=\left\lfloor\frac{2N_F}{3}\right\rfloor+1
\]
signatures. Suppose two different termination certificates for the same epoch are accepted, with signer sets \(\mathcal{Q}_1\) and \(\mathcal{Q}_2\). Then
\[
    |\mathcal{Q}_1\cap\mathcal{Q}_2|
    \ge
    |\mathcal{Q}_1|+|\mathcal{Q}_2|-N_F
    \ge
    2\tau_F-N_F.
\]
For \(N_F=3k,3k+1,3k+2\), direct substitution gives
\[
    2\tau_F-N_F>\frac{N_F}{3}.
\]
Hence \(|\mathcal{Q}_1\cap\mathcal{Q}_2|>N_F/3>B_F\), so the intersection contains at least one control-correct full node. A control-correct full node signs at most one termination certificate per epoch. Therefore two different certificates for the same epoch cannot both gather valid quorums. This proves safety.

\subsection{Proof of Theorem~\ref{theo:vv_consistency}}

If two control-correct full nodes have the same valid event-certificate and accepted-proof set up to ripple \(r\), deterministic verification gives the same vertices, self edges, gossip-parent edges, and equivocation exclusions. Therefore, the two full nodes hold the same Topology DAG prefix. By the induction argument in Lemma~\ref{lem:well_defined_voting}, the same DAG prefix gives identical moment values and reachable eligible sets. Since the vote and intermediate-confirmation rules are deterministic functions of these sets, the two full nodes compute identical \(\mathcal{V}\), \(\mathcal{U}\), and local-confirmation vectors at ripple \(r\).

\subsection{Proof of Theorem~\ref{theo:conditional_liveness}}

Because fewer than \(N_F/3\) full nodes are Byzantine, at least \(\tau_F\) full nodes are control-correct. By assumption, at ripple \(r^\star\), at least \(\tau_N\) control-correct nodes have each observed at least \(Q\) non-equivocated payload-valid origin tuples that are confirmable from the valid DAG. Eventual delivery ensures that all event certificates and accepted proofs needed for these observations eventually reach every control-correct full node. Bounded processing delay ensures that they are verified and inserted into each control-correct full node's DAG prefix in finite time.\par

By Theorem~\ref{theo:vv_consistency}, control-correct full nodes with the same decisive proof prefix compute the same confirmations for the relevant origin tuples. Once at least \(\tau_F\) matching confirmation vectors are exchanged, the same certified set \(C(r)\) is derived at control-correct full nodes. Since at least \(\tau_N\) nodes satisfy
\[
    |C(r)\cap\mathcal{O}_j(r)|\ge Q,
\]
control-correct full nodes sign termination certificates for \((t,r,C(r))\). Eventual delivery of those certificates gives every control-correct node \(\tau_F\) matching certificates, so it terminates. The proof is conditional: if the dissemination event never occurs because too many valid payloads are lost, rejected, or withheld, no protocol can force this readiness condition.

\section{Proofs of Validation and Learning Statements}
\label{app:validation_learning_proofs}

\subsection{Proof of Proposition~\ref{prop:filter_false_rejection}}

Let \(X=f_m(z)\) for a learning-clean update. Let
\[
    \mathcal{E}_n=
    \{|\mu_i^m-\bar\mu_m|\le\varepsilon_\mu,\ 
      |\sigma_i^m-\bar\sigma_m|\le\varepsilon_\sigma\}.
\]
By assumption, \(\Pr(\mathcal{E}_n)\ge1-\delta_n\). On \(\mathcal{E}_n\), upper rejection implies
\[
    X-\bar\mu_m>
    3\bar\sigma_m-\varepsilon_\mu-3\varepsilon_\sigma.
\]
For the two-sided magnitude test, lower rejection similarly implies
\[
    \bar\mu_m-X>
    3\bar\sigma_m-\varepsilon_\mu-3\varepsilon_\sigma.
\]
Let
\[
    a_m=
    \bigl(3\bar\sigma_m-\varepsilon_\mu-3\varepsilon_\sigma\bigr)_+ .
\]
The sub-Gaussian tail inequality gives
\[
    \Pr\{X-\bar\mu_m>a_m\}
    \le
    \exp\!\left(-\frac{a_m^2}{2\bar\sigma_m^2}\right),
\]
and the same bound holds for the lower tail. Therefore, by the union bound,
\[
    \Pr\{\text{clean rejection for }f_1\}
    \le
    \delta_n
    +
    2\exp\!\left(-\frac{a_1^2}{2\bar\sigma_1^2}\right).
\]
For the directional statistic, only the lower tail is used:
\[
    \Pr\{\text{clean rejection for }f_2\}
    \le
    \delta_n
    +
    \exp\!\left(-\frac{a_2^2}{2\bar\sigma_2^2}\right).
\]
This gives the proposition.

\subsection{Proof of Proposition~\ref{prop:semantic_separation}}

Let \(T_i^t=\mathrm{med}_i^t+3\,\mathrm{MAD}_i^t\). The audit removes exactly those tuples with \(d_i(z_\omega)>T_i^t\). If \(\omega\in\mathcal{H}_i^t\), then
\[
    d_i(z_\omega)\le a_i^t\le T_i^t,
\]
so it is retained. If \(\omega\in\mathcal{B}_i^t\), then
\[
    d_i(z_\omega)\ge b_i^t>T_i^t,
\]
so it is removed. The honest-majority condition ensures that the median is not controlled by Byzantine scores; exact separation follows from the displayed threshold inequality.

\subsection{Proof of Theorem~\ref{theo:learning_convergence}}

Let \(H=|\mathcal{H}_{\mathrm{L}}|\), stack the learning-clean parameters as
\[
    \Theta^t=
    [\theta_1^t,\ldots,\theta_H^t]^\top
    \in\mathbb{R}^{H\times d},
\]
and define
\[
    J=H^{-1}\bm{1}\bm{1}^{\top},
    \qquad
    P=I-J,
    \qquad
    \Omega_t=H^{-1}\EX\|P\Theta^t\|_F^2.
\]
The effective aggregation matrix \(W_t^{\mathrm{eff}}\) is row-stochastic and satisfies \eqref{eq:effective_mixing}. Let \(g_i^{t,k}\) be the stochastic gradient including the primal--dual correction. The tracking condition is \eqref{eq:tracking_assumption}.\par

After \(K\) local steps and certified aggregation, the learning-clean average obeys
\[
    \bar\theta^{t+1}
    =
    \bar\theta^t
    -
    \eta\sum_{k=0}^{K-1}\bar g^{t,k}
    +
    \eta K\bar b^t,
\]
where
\[
    \bar g^{t,k}
    =
    H^{-1}\sum_{i\in\mathcal{H}_{\mathrm{L}}}g_i^{t,k},
    \qquad
    \bar b^t=
    H^{-1}\sum_{i\in\mathcal{H}_{\mathrm{L}}}b_i^t .
\]
By Jensen's inequality,
\[
    \EX\|\bar b^t\|_2^2
    \le \delta_t^2.
\]

The disagreement recursion follows from \eqref{eq:effective_mixing}. Let \(\widetilde{\Theta}^{t+1}\) be the post-local-SGD, pre-aggregation stack. Since \(W_t^{\mathrm{eff}}\) is row-stochastic,
\[
    P W_t^{\mathrm{eff}}\widetilde{\Theta}^{t+1}
\]
is the disagreement component after effective mixing. By \eqref{eq:effective_mixing}, smoothness, bounded variance, bounded heterogeneity, and the tracking condition,
\[
\begin{aligned}
    \Omega_{t+1}
    &\le
    \frac{1+\rho}{2}\Omega_t
    +
    C_1\eta^2K^2(\sigma_g^2+\zeta^2)
    +
    C_2\eta^2K^2\delta_t^2 .
\end{aligned}
\]
Iterating the recursion gives
\[
\begin{aligned}
    \frac{1}{T}\sum_{t=0}^{T-1}\Omega_t
    &\le
    O\!\left(
    \frac{\Omega_0}{(1-\rho)T}
    \right)
    +
    O\!\left(
    \frac{\eta^2K^2(\sigma_g^2+\zeta^2)}{1-\rho}
    \right)
\\
    &\quad+
    O\!\left(
    \frac{\eta^2K^2}{(1-\rho)T}
    \sum_{t=0}^{T-1}\delta_t^2
    \right).
\end{aligned}
\]

Define
\[
    G^t=K^{-1}\sum_{k=0}^{K-1}\bar g^{t,k},
    \qquad
    e^t=G^t-\nabla F_{\mathcal{H}_{\mathrm{L}}}(\bar\theta^t).
\]
Using \(L\)-smoothness, gradient variance, heterogeneity, disagreement, local drift, and tracking,
\[
    \EX\|e^t\|_2^2
    \le
    C_3
    \left(
    \frac{\sigma_g^2}{HK}
    +
    L^2\Omega_t
    +
    \eta^2K^2L^2(\sigma_g^2+\zeta^2)
    +
    \delta_t^2
    \right).
\]
By \(L\)-smoothness of \(F_{\mathcal{H}_{\mathrm{L}}}\),
\[
\begin{aligned}
F_{\mathcal{H}_{\mathrm{L}}}(\bar\theta^{t+1})
&\le
F_{\mathcal{H}_{\mathrm{L}}}(\bar\theta^t)
+
\left\langle
\nabla F_{\mathcal{H}_{\mathrm{L}}}(\bar\theta^t),
-\eta K G^t+\eta K\bar b^t
\right\rangle
\\
&\quad+
\frac{L}{2}
\left\|
-\eta K G^t+\eta K\bar b^t
\right\|_2^2 .
\end{aligned}
\]
Substituting \(G^t=\nabla F_{\mathcal{H}_{\mathrm{L}}}(\bar\theta^t)+e^t\), applying Young's inequality, and choosing \(\eta K L\) sufficiently small gives
\[
\begin{aligned}
\EX F_{\mathcal{H}_{\mathrm{L}}}(\bar\theta^{t+1})
&\le
\EX F_{\mathcal{H}_{\mathrm{L}}}(\bar\theta^t)
\\
&\quad
-\frac{\eta K}{4}
\EX\!\left[
\left\|
\nabla F_{\mathcal{H}_{\mathrm{L}}}(\bar\theta^t)
\right\|_2^2
\right]
\\
&\quad
+C_4\eta K\EX\|e^t\|_2^2
+C_5\eta K\delta_t^2 .
\end{aligned}
\]
Summing from \(0\) to \(T-1\), using the lower bound \(F_{\inf}\), and substituting the bounds on \(e^t\) and \(\Omega_t\), yields
\[
\begin{aligned}
\frac{1}{T}
\sum_{t=0}^{T-1}
\EX\!\left[
\left\|
\nabla F_{\mathcal{H}_{\mathrm{L}}}(\bar\theta^t)
\right\|_2^2
\right]
&\le
O\!\left(
\frac{
F_{\mathcal{H}_{\mathrm{L}}}(\bar\theta^0)-F_{\inf}
}{
\eta KT
}
\right)
\\
&\quad+
O(\eta\sigma_g^2)
+
O\!\left(
\frac{\eta K\zeta^2}{(1-\rho)^2}
\right)
\\
&\quad+
O\!\left(
\frac{1}{T}
\sum_{t=0}^{T-1}
\delta_t^2
\right).
\end{aligned}
\]
The \((1-\rho)^{-2}\) factor is the standard amplification of local drift and heterogeneity through decentralized mixing. Taking \(\eta=\Theta(T^{-1/2})\) gives the stated stationary-neighborhood conclusion.

\endgroup

\balance
{\scriptsize
\bibliographystyle{IEEEtran}

\begin{thebibliography}{99}
\setlength{\itemsep}{0pt}

\bibitem{McMahan2017}
H.~B. McMahan, E.~Moore, D.~Ramage, S.~Hampson, and B.~Ag{\"u}era y Arcas, ``Communication-efficient learning of deep networks from decentralized data,'' in \textit{Proc. 20th Int. Conf. Artif. Intell. Statist. (AISTATS)}, 2017, pp.~1273--1282.

\bibitem{Kairouz2021}
P.~Kairouz \textit{et al.}, ``Advances and open problems in federated learning,'' \textit{Found. Trends Mach. Learn.}, vol.~14, no.~1--2, pp.~1--210, 2021, doi: 10.1561/2200000083.

\bibitem{Taherpour2026ZKHybridFL}
A.~Taherpour and X.~Wang, ``ZK-HybridFL: Zero-knowledge proof-enhanced hybrid ledger for federated learning,'' \textit{IEEE Trans. Neural Netw. Learn. Syst.}, early access, pp.~1--15, Feb.~2026, doi: 10.1109/TNNLS.2026.3658993.

\bibitem{Bonawitz2017}
K.~Bonawitz \textit{et al.}, ``Practical secure aggregation for privacy-preserving machine learning,'' in \textit{Proc. ACM SIGSAC Conf. Comput. Commun. Secur. (CCS)}, 2017, pp.~1175--1191, doi: 10.1145/3133956.3133982.

\bibitem{Zhu2019}
L.~Zhu, Z.~Liu, and S.~Han, ``Deep leakage from gradients,'' in \textit{Proc. Adv. Neural Inf. Process. Syst. (NeurIPS)}, 2019, pp.~14747--14756.

\bibitem{Boyd2006}
S.~Boyd, A.~Ghosh, B.~Prabhakar, and D.~Shah, ``Randomized gossip algorithms,'' \textit{IEEE Trans. Inf. Theory}, vol.~52, no.~6, pp.~2508--2530, Jun.~2006, doi: 10.1109/TIT.2006.874516.

\bibitem{Lian2017}
X.~Lian, C.~Zhang, H.~Zhang, C.-J. Hsieh, W.~Zhang, and J.~Liu, ``Can decentralized algorithms outperform centralized algorithms? A case study for decentralized parallel stochastic gradient descent,'' in \textit{Proc. Adv. Neural Inf. Process. Syst. (NeurIPS)}, 2017, pp.~5330--5340.

\bibitem{Lian2018}
X.~Lian, W.~Zhang, C.-J. Hsieh, C.~Zhang, and J.~Liu, ``Asynchronous decentralized parallel stochastic gradient descent,'' in \textit{Proc. Int. Conf. Mach. Learn. (ICML)}, 2018, pp.~3043--3052.

\bibitem{Assran2019}
M.~Assran, N.~Loizou, N.~Ballas, and M.~G. Rabbat, ``Stochastic gradient push for distributed deep learning,'' in \textit{Proc. Int. Conf. Mach. Learn. (ICML)}, 2019, pp.~344--353.

\bibitem{Koloskova2019}
A.~Koloskova, S.~Stich, and M.~Jaggi, ``Decentralized stochastic optimization and gossip algorithms with compressed communication,'' in \textit{Proc. Int. Conf. Mach. Learn. (ICML)}, 2019, pp.~3478--3487.

\bibitem{Tang2023}
Z.~Tang, S.~Shi, B.~Li, and X.~Chu, ``GossipFL: A decentralized federated learning framework with sparsified and adaptive communication,'' \textit{IEEE Trans. Parallel Distrib. Syst.}, vol.~34, no.~3, pp.~909--922, Mar.~2023, doi: 10.1109/TPDS.2022.3230938.

\bibitem{Chen2023}
Q.~Chen, Z.~Wang, H.~Wang, and X.~Lin, ``FedDual: Pair-wise gossip helps federated learning in large decentralized networks,'' \textit{IEEE Trans. Inf. Forensics Security}, vol.~18, pp.~335--350, 2023, doi: 10.1109/TIFS.2022.3222935.

\bibitem{Wang2025}
H.~Wang and Y.~Chi, ``Communication-efficient federated optimization over semi-decentralized networks,'' \textit{IEEE Trans. Signal Inf. Process. Netw.}, vol.~11, pp.~147--160, 2025, doi: 10.1109/TSIPN.2025.3539004.

\bibitem{Taherpour2025SPIDChain}
A.~Taherpour and X.~Wang, ``SPID-Chain: A smart contract-enabled, polar-coded interoperable DAG chain,'' arXiv:2501.11794, Jan.~2025, doi: 10.48550/arXiv.2501.11794.

\bibitem{Blanchard2017}
P.~Blanchard, E.~M. El~Mhamdi, R.~Guerraoui, and J.~Stainer, ``Machine learning with adversaries: Byzantine tolerant gradient descent,'' in \textit{Proc. Adv. Neural Inf. Process. Syst. (NeurIPS)}, 2017, pp.~119--129.

\bibitem{Yin2018}
D.~Yin, Y.~Chen, R.~Kannan, and P.~Bartlett, ``Byzantine-robust distributed learning: Towards optimal statistical rates,'' in \textit{Proc. Int. Conf. Mach. Learn. (ICML)}, 2018, pp.~5650--5659.

\bibitem{Pillutla2022}
K.~Pillutla, S.~M. Kakade, and Z.~Harchaoui, ``Robust aggregation for federated learning,'' \textit{IEEE Trans. Signal Process.}, vol.~70, pp.~1142--1154, 2022, doi: 10.1109/TSP.2022.3153135.

\bibitem{Fung2018}
C.~Fung, C.~J.~M. Yoon, and I.~Beschastnikh, ``Mitigating sybils in federated learning poisoning,'' arXiv:1808.04866, 2018, doi: 10.48550/arXiv.1808.04866.

\bibitem{Bagdasaryan2020}
E.~Bagdasaryan, A.~Veit, Y.~Hua, D.~Estrin, and V.~Shmatikov, ``How to backdoor federated learning,'' in \textit{Proc. Int. Conf. Artif. Intell. Statist. (AISTATS)}, 2020, pp.~2938--2948.

\bibitem{Wang2020}
H.~Wang \textit{et al.}, ``Attack of the tails: Yes, you really can backdoor federated learning,'' in \textit{Proc. Adv. Neural Inf. Process. Syst. (NeurIPS)}, 2020, pp.~16070--16084.

\bibitem{Wang2019a}
B.~Wang, Y.~Yao, S.~Shan, H.~Li, B.~Viswanath, H.~Zheng, and B.~Y. Zhao, ``Neural Cleanse: Identifying and mitigating backdoor attacks in neural networks,'' in \textit{Proc. IEEE Symp. Secur. Privacy (S\&P)}, 2019, pp.~707--723, doi: 10.1109/SP.2019.00031.

\bibitem{STRIP2019}
Y.~Gao, C.~Xu, D.~Wang, S.~Chen, D.~C. Ranasinghe, and S.~Nepal, ``STRIP: A defence against Trojan attacks on deep neural networks,'' in \textit{Proc. 35th Annu. Comput. Secur. Appl. Conf. (ACSAC)}, 2019, pp.~113--125, doi: 10.1145/3359789.3359790.

\bibitem{SpectralSignatures2018}
B.~Tran, J.~Li, and A.~Madry, ``Spectral signatures in backdoor attacks,'' in \textit{Proc. Adv. Neural Inf. Process. Syst. (NeurIPS)}, 2018, pp.~8011--8021.

\bibitem{Taherpour2025CodedBlockchainIoT}
A.~Taherpour and X.~Wang, ``A high-throughput and secure coded blockchain for IoT,'' \textit{IEEE Trans. Dependable Secure Comput.}, vol.~22, no.~4, pp.~3561--3579, Jul.--Aug.~2025, doi: 10.1109/TDSC.2025.3532850.

\bibitem{Taherpour2024HybridChain}
A.~Taherpour and X.~Wang, ``HybridChain: Fast, accurate, and secure transaction processing with distributed learning,'' \textit{IEEE Trans. Parallel Distrib. Syst.}, vol.~35, no.~6, pp.~968--982, Jun.~2024, doi: 10.1109/TPDS.2024.3381593.

\bibitem{Kim2019}
H.~Kim, J.~Park, M.~Bennis, and S.-L. Kim, ``Blockchained on-device federated learning,'' \textit{IEEE Commun. Lett.}, vol.~24, no.~6, pp.~1279--1283, Jun.~2020, doi: 10.1109/LCOMM.2019.2921755.

\bibitem{Warnat2021}
S.~Warnat-Herresthal \textit{et al.}, ``Swarm learning for decentralized and confidential clinical machine learning,'' \textit{Nature}, vol.~594, no.~7862, pp.~265--270, Jun.~2021, doi: 10.1038/s41586-021-03583-3.

\bibitem{Qu2020}
Y.~Qu \textit{et al.}, ``Decentralized privacy using blockchain-enabled federated learning in fog computing,'' \textit{IEEE Internet Things J.}, vol.~7, no.~6, pp.~5171--5183, Jun.~2020, doi: 10.1109/JIOT.2020.2977383.

\bibitem{Li2022}
J.~Li \textit{et al.}, ``Blockchain assisted decentralized federated learning (BLADE-FL): Performance analysis and resource allocation,'' \textit{IEEE Trans. Parallel Distrib. Syst.}, vol.~33, no.~10, pp.~2401--2415, Oct.~2022, doi: 10.1109/TPDS.2021.3138848.

\bibitem{jkmn1}
Z.~Cai, J.~Chen, Y.~Fan, Z.~Zheng, and K.~Li, ``Blockchain-empowered federated learning: Benefits, challenges, and solutions,'' \textit{IEEE Trans. Big Data}, vol.~11, no.~5, pp.~2244--2263, Oct.~2025, doi: 10.1109/TBDATA.2025.3541560.

\bibitem{bambool1}
Z.~Peng \textit{et al.}, ``VFChain: Enabling verifiable and auditable federated learning via blockchain systems,'' \textit{IEEE Trans. Netw. Sci. Eng.}, vol.~9, no.~1, pp.~173--186, Jan.--Feb.~2022, doi: 10.1109/TNSE.2021.3050781.

\bibitem{WWWZ1}
A.~P. Kalapaaking, I.~Khalil, X.~Yi, K.-Y. Lam, G.-B. Huang, and N.~Wang, ``Auditable and verifiable federated learning based on blockchain-enabled decentralization,'' \textit{IEEE Trans. Neural Netw. Learn. Syst.}, vol.~36, no.~1, pp.~102--115, Jan.~2025, doi: 10.1109/TNNLS.2024.3407670.

\bibitem{Yuan2024}
S.~Yuan, B.~Cao, Y.~Sun, Z.~Wan, and M.~Peng, ``Secure and efficient federated learning through layering and sharding blockchain,'' \textit{IEEE Trans. Netw. Sci. Eng.}, vol.~11, no.~3, pp.~3120--3134, May--Jun.~2024, doi: 10.1109/TNSE.2024.3361458.

\bibitem{Cao2023}
M.~Cao, L.~Zhang, and B.~Cao, ``Toward on-device federated learning: A direct acyclic graph-based blockchain approach,'' \textit{IEEE Trans. Neural Netw. Learn. Syst.}, vol.~34, no.~4, pp.~2028--2042, Apr.~2023, doi: 10.1109/TNNLS.2021.3105810.

\bibitem{Chen2025}
J.~Chen, D.~Wu, S.~Guo, F.~Qi, and X.~Qiu, ``DAG-EnseFL: DAG-based asynchronous federated learning with ensemble distillation,'' \textit{IEEE Trans. Big Data}, vol.~11, no.~6, pp.~3342--3355, Dec.~2025, doi: 10.1109/TBDATA.2025.3594244.

\bibitem{Wang2025DAG}
Q.~Wang, S.~Xu, R.~Xu, and B.~Ai, ``A DAG-blockchain-assisted federated learning framework in wireless networks: Learning performance and throughput optimization schemes,'' \textit{IEEE Trans. Veh. Technol.}, vol.~74, no.~3, pp.~5097--5113, Mar.~2025, doi: 10.1109/TVT.2024.3502444.

\bibitem{BB6}
G.~Yu \textit{et al.}, ``IronForge: An open, secure, fair, decentralized federated learning,'' \textit{IEEE Trans. Neural Netw. Learn. Syst.}, vol.~36, no.~1, pp.~354--368, Jan.~2025, doi: 10.1109/TNNLS.2023.3329249.

\bibitem{Baird2016}
L.~Baird, ``The Swirlds Hashgraph consensus algorithm: Fair, fast, Byzantine fault tolerance,'' Swirlds, Tech. Rep. SWIRLDS-TR-2016-01, 2016. [Online]. Available: \url{https://www.swirlds.com/downloads/SWIRLDS-TR-2016-01.pdf}. Accessed: Jul.~9, 2026.

\bibitem{Castro1999}
M.~Castro and B.~Liskov, ``Practical Byzantine fault tolerance,'' in \textit{Proc. 3rd Symp. Operating Syst. Design Implement. (OSDI)}, 1999, pp.~173--186.

\bibitem{FLTrust2021}
X.~Cao, M.~Fang, J.~Liu, and N.~Z. Gong, ``FLTrust: Byzantine-robust federated learning via trust bootstrapping,'' in \textit{Proc. Netw. Distrib. Syst. Secur. Symp. (NDSS)}, 2021, doi: 10.14722/ndss.2021.24434.

\bibitem{ZenoPP2020}
C.~Xie, O.~Koyejo, and I.~Gupta, ``Zeno++: Robust fully asynchronous SGD,'' in \textit{Proc. Int. Conf. Mach. Learn. (ICML)}, ser. \textit{Proc. Mach. Learn. Res.}, vol.~119, 2020, pp.~10495--10503.

\bibitem{Iglewicz1993}
B.~Iglewicz and D.~C. Hoaglin, \textit{How to Detect and Handle Outliers}. Milwaukee, WI, USA: ASQC Quality Press, 1993.

\bibitem{Leys2013}
C.~Leys, C.~Ley, O.~Klein, P.~Bernard, and L.~Licata, ``Detecting outliers: Do not use standard deviation around the mean, use absolute deviation around the median,'' \textit{J. Exp. Soc. Psychol.}, vol.~49, no.~4, pp.~764--766, Jul.~2013, doi: 10.1016/j.jesp.2013.03.013.

\bibitem{RFC8032}
S.~Josefsson and I.~Liusvaara, ``Edwards-Curve Digital Signature Algorithm (EdDSA),'' RFC 8032, Jan.~2017, doi: 10.17487/RFC8032.

\bibitem{Ed25519Bernstein2012}
D.~J. Bernstein, N.~Duif, T.~Lange, P.~Schwabe, and B.-Y. Yang, ``High-speed high-security signatures,'' \textit{J. Cryptographic Eng.}, vol.~2, no.~2, pp.~77--89, Sep.~2012, doi: 10.1007/s13389-012-0027-1.

\bibitem{dagsim_ref}
B.~Schachenhofer, ``dagsim: Hashgraph simulator,'' GitHub repository. [Online]. Available: \url{https://github.com/BSchachenhofer/dagsim}. Accessed: Jul.~9, 2026.

\bibitem{jgrapht_ref}
D.~Michail, J.~Kinable, B.~Naveh, and J.~V. Sichi, ``JGraphT---A Java library for graph data structures and algorithms,'' \textit{ACM Trans. Math. Softw.}, vol.~46, no.~2, Art.~no.~16, May~2020, doi: 10.1145/3381449.

\bibitem{bouncycastle_ref}
The Legion of the Bouncy Castle Inc., ``Bouncy Castle Crypto APIs,'' [Online]. Available: \url{https://www.bouncycastle.org/}. Accessed: Jul.~9, 2026.

\bibitem{pytorch_ref}
A.~Paszke \textit{et al.}, ``PyTorch: An imperative style, high-performance deep learning library,'' in \textit{Proc. Adv. Neural Inf. Process. Syst. (NeurIPS)}, 2019, pp.~8024--8035.

\bibitem{numpy_ref}
C.~R. Harris \textit{et al.}, ``Array programming with NumPy,'' \textit{Nature}, vol.~585, no.~7825, pp.~357--362, Sep.~2020, doi: 10.1038/s41586-020-2649-2.

\bibitem{grpckotlin_ref}
The gRPC Authors, ``gRPC-Kotlin,'' GitHub repository. [Online]. Available: \url{https://github.com/grpc/grpc-kotlin}. Accessed: Jul.~9, 2026.

\bibitem{grpcio_ref}
The gRPC Authors, ``grpcio: gRPC for Python,'' PyPI. [Online]. Available: \url{https://pypi.org/project/grpcio/}. Accessed: Jul.~9, 2026.

\bibitem{protobuf_ref}
Google, ``Protocol Buffers documentation,'' [Online]. Available: \url{https://protobuf.dev/}. Accessed: Jul.~9, 2026.

\bibitem{matplotlib_ref}
J.~D. Hunter, ``Matplotlib: A 2D graphics environment,'' \textit{Comput. Sci. Eng.}, vol.~9, no.~3, pp.~90--95, May--Jun.~2007, doi: 10.1109/MCSE.2007.55.

\bibitem{adpsgd_repo}
Facebook Research, ``stochastic\_gradient\_push: PyTorch implementation of Stochastic Gradient Push,'' GitHub repository. [Online]. Available: \url{https://github.com/facebookresearch/stochastic_gradient_push}. Accessed: Jul.~9, 2026.

\bibitem{blade_repo}
Y.-M. Shao, ``BLADE-FL: Blockchain Assisted Decentralized Federated Learning,'' GitHub repository. [Online]. Available: \url{https://github.com/ElvisShaoYumeng/BLADE-FL}. Accessed: Jul.~9, 2026.

\bibitem{ganache_ref}
Truffle Suite, ``Ganache,'' [Online]. Available: \url{https://archive.trufflesuite.com/ganache/}. Accessed: Jul.~9, 2026.

\bibitem{gggg1}
S.~Yuan, ``ChainsFL: A blockchain-based federated learning implementation,'' GitHub repository, 2021. [Online]. Available: \url{https://github.com/shuoyuan/ChainsFL-implementation}. Accessed: Jul.~9, 2026.

\bibitem{mnist_ref}
Y.~LeCun, C.~Cortes, and C.~J.~C. Burges, ``The MNIST database of handwritten digits,'' [Online]. Available: \url{http://yann.lecun.com/exdb/mnist/}. Accessed: Jul.~9, 2026.

\bibitem{gggg12}
M.~Sandler, A.~Howard, M.~Zhu, A.~Zhmoginov, and L.-C. Chen, ``MobileNetV2: Inverted residuals and linear bottlenecks,'' in \textit{Proc. IEEE/CVF Conf. Comput. Vis. Pattern Recognit. (CVPR)}, 2018, pp.~4510--4520, doi: 10.1109/CVPR.2018.00474.

\bibitem{penn_treebank_ref}
M.~P. Marcus, B.~Santorini, and M.~A. Marcinkiewicz, ``Building a large annotated corpus of English: The Penn Treebank,'' \textit{Comput. Linguistics}, vol.~19, no.~2, pp.~313--330, 1993.

\bibitem{networkx_ref}
A.~A. Hagberg, D.~A. Schult, and P.~J. Swart, ``Exploring network structure, dynamics, and function using NetworkX,'' in \textit{Proc. Python Sci. Conf. (SciPy)}, 2008, pp.~11--15.

\end{thebibliography}

}
\end{document}